\documentclass[a4paper,10pt]{article}

\usepackage{publication}
\usepackage{caption}
\usepackage{amsmath, amsthm}
\usepackage{amssymb}
\usepackage{graphicx}
\usepackage{booktabs}
\usepackage{multirow, multicol}
\usepackage{ulem}
\usepackage{adjustbox}
\usepackage[backend=biber, natbib=true]{biblatex}
\newtheorem{theorem}{Theorem}

\title{MARVEL: Margin-Aware Robust von Mises-Fischer Expert Learning for Long-Tailed Out-of-Distribution Detection}

\author{A.S. Anudeep}
\contact{archakams@iisc.ac.in}
\affil{Department of Computational and Data Sciences, Indian Institute of Science}

\author{Vaanathi Sundaresan}
\contact{vaanathi@iisc.ac.in}

\date{\today}

\begin{document}

\maketitle

\begin{abstract}
For clinical deployment, it is essential that automated diagnostic systems remain reliable when confronted with previously unseen cases, yet deep models routinely misclassify out-of-distribution (OOD) inputs with high confidence, underscoring the need for more robust OOD detection methods. Although substantial effort has been devoted to improving model robustness, most of the existing literature assumes balanced datasets, evaluates OOD detection on coarse or non-clinical OOD sources, or lacks comprehensive assessment across diverse OOD scenarios. To address the gaps, we propose a novel  methodology trained on diverse and imbalanced medical datasets and evaluated across a clinically reflective OOD spectrum. Our framework comprises three key components: (1) a Nonlinear von Mises-Fisher (NvMF) classifier capable of learning non-linear decision boundaries, with theoretical proof of its asymptotic connection to cosine classifiers; (2) a multi-expert framework in which margin-aware NvMF classifiers specialise in different regions of label distribution to better handle imbalance; and (3) an outlier expert trained explicitly to distinguish inlier from outlier data, thereby strengthening OOD detection. Evaluation on RFMiD, ISIC2019, and NCTCRC datasets demonstrates consistent improvements over state-of-the-art methods, achieving mean FPR95 reductions of 8.45\%, 13.02\%, and 36.90\% respectively. 
These gains are further supported by comprehensive ablations that validated the contributions of each component. This enables reliable identification of unfamiliar cases for deferral to clinicians, supporting safer AI-assisted diagnosis in real-world  workflows. Our code is available at \url{https://github.com/redboxup/MARVEL}.
\end{abstract}



\begin{multicols}{2}

\section{Introduction}\label{sec:introduction}
Deep neural networks (DNNs) have demonstrated promising performance across a wide range of medical imaging applications including disease classification and region-of-interest detection.
DNNs have often achieved performance on par to or at times even exceeding clinicians on standard benchmarks \citep{liu2019comparison,9363915, Esteva2017Dermatologist}. However, a typical in-distribution (ID) classification task using DNN assume a closed and balanced label space, whereas real-world data, especially in medical imaging, follows a long-tail distribution with many rare, atypical, or previously unseen cases.
Out-of-distribution (OOD) data refers to inputs whose underlying distribution deviates from that of those seen in the training set.
For example, for a DNN trained to classify pathologies in frontal chest X-rays, unseen/OOD cases include lateral views or X-rays from entirely different anatomical regions or even images that are not radiographs at all.
In the above case, lateral views constitute the \textit{nearOOD} data (same modality with relatively smaller shift from training distribution), while non-radiographic images form the \textit{farOOD} data.
DNN models often produce overconfident erroneous predictions on these OOD samples, by assigning incorrect ID labels, and often with an unjustifiably high softmax probability \citep{yang2024oodsurvey, Wollek2024OODVoting}.
This poses significant safety and reliability risks, making the vulnerability of DNNs to OOD data one of their main limitations. 
Adding to this challenge is the class imbalance, often prevalent in real-world medical datasets, where a few common classes dominate the dataset while numerous rare classes form a long tail \citep{cui2019classbalancedloss, LITJENS201760}.
The frequent classes are typically referred to as \textit{head classes}, whereas the rare classes are known as \textit{tail classes}.
Long-tailed distributions bias learning towards head classes, affecting representation learning for infrequent classes, causing their features to be poorly separated and closer to decision boundaries \citep{cao2019learning, menon2021longtail}.
This leads to the model being unable to distinguish tail data classes from the true OOD data, causing frequent misclassifications of rare cases as OOD and vice versa.

To address this particular limitation, there has been growing interest in studying OOD detection under long-tailed distributions.
Recent methods introduce tailored contrastive objectives  \citep{wang2022partial, he2024longtailedoutofdistributiondetectionprioritizing}, logit-adjusted calibration \citep{miao2024out, wei2024EAT}, multi-outlier modelling \citep{GUHAROY2022102274, wei2024EAT}, hyperspherical embedding regularization \citep{he2024longtailedoutofdistributiondetectionprioritizing, openlongtailrecognition}, and distribution-aware reweighting schemes  \citep{wei2024EAT,he2024longtailedoutofdistributiondetectionprioritizing,miao2024out}, combined with auxiliary OOD data \citep{wang2022partial,miao2024out, wei2024EAT,he2024longtailedoutofdistributiondetectionprioritizing}, to jointly improve tail recognition and OOD robustness.
To assess the effectiveness of these approaches, existing efforts construct long-tailed variants of standard ID datasets and evaluate against a fixed set of OOD benchmarks.
While this strategy has contributed valuable insights, it has primarily focused on natural-image settings.
As a consequence, the selected OOD datasets do not reflect the diverse challenges present in medical imaging, where distribution shifts may arise from multiple causes including modality differences, variations in acquisition protocol, institutional heterogeneity, demographic diversity, the presence of rare or novel pathologies.

Motivated by these gaps, we propose a method for medical imaging ID classification with OOD detection, under long-tailed class distributions, \texttt{MARVEL} (\textbf{M}argin-\textbf{A}ware \textbf{R}obust \textbf{v}on Mises–Fisher \textbf{E}xpert \textbf{L}earning). At the core of \texttt{MARVEL} is a Nonlinear von Mises–Fisher (NvMF) classifier that generalises conventional vMF-likelihood formulations.
Prior work typically treats the vMF distribution as a class-conditional density and reduces its logits to scaled cosine similarities, which limits the decision boundary flexibility and prevents modelling of complex, heterogeneous class geometries.
In contrast, we reinterpret the vMF distribution through its exponential-family structure and define logits using changes in the log-partition function.
This perspective yields a principled formulation that induces non-linear decision boundaries, enabling richer and more flexible separation between classes.
However, enhancing decision boundaries alone does not fully mitigate the difficulties imposed by long-tailed imbalance. 
Even with non-linear boundaries, rare classes continue to suffer from overlapping feature clusters and insufficient decision margins.
To counter this, we introduce a margin-aware multi-expert architecture comprising three variants (experts) of our proposed classifier, each trained under distinct prior assumptions: head-biased, balanced, and tail-biased.
Each expert specialises in a different region of the label distribution, effectively capturing the heterogeneity induced by severe imbalance.
In addition, to explicitly model abnormal, unseen, or clinically irrelevant inputs, we incorporate a dedicated outlier expert trained solely to discriminate OOD from ID samples.

Beyond architectural improvements, we also reexamine how OOD detection is evaluated in medical imaging.
Existing benchmarks often suffer from two key limitations: (1) medical imaging datasets typically restrict evaluation to nearOOD scopes, such as only holding out novel classes, which overlooks broader distribution-shift scenarios; and (2) natural image benchmarks frequently rely on coarse-grained OOD datasets like SVHN \citep{Netzer2011}, LSUN \citep{yu15lsun}, and INAT \citep{inat2024}, which fail to reflect the complexity, variability, and multimodal nature of OOD data from medical imaging.
To address these limitations, we assemble a diverse collection of OOD datasets spanning novel classes, corruption-based shifts, and farOOD natural images, establishing one of the first benchmarks that systematically evaluates medical-imaging models across both near and broad distribution shifts. 

In summary, our work contributes both methodological advances and comprehensive evaluation framework for medical imaging OOD detection as follows:

\begin{itemize}

\item \textbf{Nonlinear vMF classifier:} By leveraging the exponential family structure of the vMF distribution, we derive a novel Nonlinear vMF-based classifier that admits a more adaptive and complex decision boundaries in hyperspherical representation spaces, and theoretically show that it converges to a cosine classifier in the asymptotic regime.

\item \textbf{Margin-aware multi-expert learning strategy:} We develop a multi-expert ensemble composed of margin-aware NvMF classifiers that specialise across different regions of the label frequency distribution.

\item \textbf{Dedicated outlier expert for explicit OOD modelling:} We design a dedicated outlier expert that is trained to explicitly separate OOD samples from ID inputs, enhancing robustness to unseen and anomalous cases.

\item \textbf{Clinically realistic OOD evaluation: } We introduce an OOD evaluation framework for medical imaging that captures both semantic and domain-level shifts, going beyond prior benchmarks that rely on artificially constructed OOD datasets (not representative of real-world OOD cases), to reflect clinically relevant failure modes. 

\item \textbf{Evaluation on diverse medical datasets:} We conduct extensive experiments on multi-modal and multi-organ datasets including NCTCRC for histopathology, ISIC2019 for dermatology, and RFMiD for retinal imaging.

\item \textbf{Comprehensive empirical validation:} We conduct ablation studies on the proposed methodology, complemented by calibration, risk-coverage, and qualitative analyses, to rigorously substantiate the effectiveness and robustness of the proposed method.

\end{itemize}

\section{Related Work}
\label{relwork}
OOD detection has been so far explored in generic computer vision and machine learning, spanning diverse assumptions, methodologies, and evaluation protocols.
However, relatively little attention has been paid to identifying approaches that remain effective in long-tailed medical imaging settings.
To provide a coherent context for our framework, we first summarise the core developments in classification-based OOD detection (the setting most directly aligned with our problem formulation) before discussing recent efforts that explicitly address distribution imbalance and long-tailed regimes.

It is important to note that OOD detection also encompasses several specialised techniques, including
generative-model-based approaches for detecting pixel-level anomalies or spatially localised irregularities, as well as methods tailored for object detection or segmentation pipelines.
While effective in their respective applications, such approaches target different problem formulations from the image-level classification setting considered here, and are therefore excluded from our review.

\subsection{Out-of-Distribution Detection}
Existing approaches can be broadly categorised into \textit{post-hoc} methods, which operate on pretrained models at inference, and \textit{training-time} methods, which explicitly shape representations or decision boundaries during learning.

\textbf{Post-hoc Methods:}
Post-hoc approaches derive OOD scores from a fixed pretrained model without modifying its parameters.
A large class of methods relies on classifier outputs to estimate confidence.
Early work such as Maximum Softmax Probability (MSP) \citep{hendrycks17baseline} uses the maximum softmax score as a proxy for confidence, while Maximum Logits Score (MLS) \citep{vaze2022openset} operates directly on raw logits to provide a more stable signal.
Energy-based methods \citep{liu2020energy} further improve this by using exponentiated logits, summed across classes, to capture the overall evidence distribution.
Subsequent works enhance score-based methods via calibration and perturbation strategies.
For instance, ODIN \citep{liang2018enhancing} introduces temperature scaling and input perturbations to amplify the separation between ID and OOD samples, while its generalised variant G-ODIN \citep{Hsu_2020_CVPR} extend this idea to broader settings.
Further, methods such as ViM \citep{haoqi2022vim} augment the logit space with a virtual OOD component derived from feature residuals, bridging output-based and feature-based perspectives.
Beyond output scores, another line of work focuses on modifying intermediate activations to improve OOD separability.
Activation shaping methods such as REACT \citep{sun2021react} truncate overly large activations, while ASH \citep{djurisic2023extremely}  prune activations to suppress spurious high responses.
Similarly, DICE \citep{sun2022dice} and SCALE \citep{xu2024scaling} selectively retain or rescale salient activations, leading to improved discrimination without retraining.
These methods are lightweight and model-agnostic, making them attractive for practical deployment.

Another class of post-hoc methods are distance-based detectors, which leverage feature embeddings extracted from the penultimate layer of deep encoders.
These methods assume that the learned embedding space exhibits sufficient structure to separate ID and OOD samples.
Under this assumption, Mahalanobis Distance (MD) based detector \citep{NEURIPS2018_abdeb6f5,sehwag2021ssd}, models class-conditional feature distributions as Gaussians and computes the Mahalanobis distance between a test sample and the estimated distribution, treating samples that lie far from all class centres as potential OOD inputs. 
Similarly, k-th nearest-neighbour based method \citep{sun2022knnood} measures the distance between a test sample and its k-th nearest neighbour in the feature space, assuming that ID samples will lie close to previously observed representations while OOD samples will be more isolated.
Hybrid approaches such as NN-Guide \citep{Park_2023_ICCV} combine the classifier scores and feature distances to try and improve both near- and farOOD detection.

\textbf{Training-time Methods:}
In contrast to post-hoc approaches, training-time methods aim to learn representations that are inherently more amenable to OOD detection.

A prominent class of these methods leverages contrastive learning to enforce structured feature spaces.
SSD \citep{sehwag2021ssd} employs supervised contrastive objectives to promote intra-class compactness and inter-class separation, improving downstream distance-based detection.
CIDER \citep{ming2023cider} extends this idea by modelling normalised features on a hypersphere using the vMF distribution, encouraging concentration around class-specific directions.
Building further, PALM \citep{lu2024learning} relaxes the single-prototype assumption by learning multiple prototypes per class, enabling better modelling of intra-class variability while preserving discriminative boundaries.
Another widely explored direction incorporates auxiliary outlier data during training.
Outlier Exposure (OE) \citep{hendrycks2019oe} trains models to assign low confidence to samples drawn from external datasets, with the expectation that this behaviour generalises to unseen OOD inputs.
Several methods have adopted this principle, incorporating auxiliary outlier data into the training process to enhance OOD detection \citep{Wollek2024OODVoting, Zhang_2023_WACV, Yang_2021_ICCV, pmlr-v162-ming22a}.
Another line of work extends this idea by synthesising outliers during training.
For example, methods such as VOS \citep{du2022vos} and NPOS \citep{tao2023nonparametric} generate virtual outliers to improve robustness without requiring curated external data.
Collectively, these approaches promote structured and discriminative feature representations, which in turn improve the OOD discrimination of post-hoc detectors.

Despite their effectiveness, existing post-hoc and training-time approaches exhibit fundamental limitations in long-tailed settings due to inherent biases in both representation learning and decision mechanisms.
Confidence-based methods suffer from skewed classifier outputs, where models trained on imbalanced data tend to be overconfident on head classes while producing unreliable scores for tail classes, often causing tail samples to be misidentified as OOD.
Activation shaping methods, while effective in suppressing spurious high activations, rely on global heuristics that are agnostic to class structure, and may inadvertently remove informative but low-magnitude signals from underrepresented classes, thereby degrading discrimination for tail samples.
Distance-based approaches face a more structural issue, as tail class representations are typically sparse, anisotropic, and biased toward head-class directions, making distance estimates unstable and less discriminative.
Although contrastive learning methods improve global feature separability, they are driven by the empirical data distribution and thus disproportionately favour head classes, failing to enforce balanced representation quality across the label space.
Similarly, outlier exposure introduces competing objectives that can divert model capacity toward separating ID from OOD at the expense of refining fine-grained boundaries for underrepresented classes.
More broadly, these methods treat OOD detection and class imbalance as independent challenges, this mismatch highlights the need for approaches that explicitly account for the interplay between imbalance and distributional uncertainty.

\subsection{Long-Tailed Out-of-Distribution Detection}

An early step toward long-tailed OOD detection is Open Long Tail Recognition (OLTR) \citep{openlongtailrecognition}, which employs a meta-embedding framework to model tail-class prototypes and reject samples that deviate from them.
The Hierarchical Outlier Detection (HOD) loss \citep{GUHAROY2022102274}, in contrast, treats tail classes as surrogate outliers to enforce inlier–outlier separation.
Similarly, \citet{miccai-ood-long-tail} address long-tailed and fine-grained OOD detection using a combination of mixup strategies and prototype learning within a skin-lesion setting.
Consequently, all approaches reduce OOD detection to a form of novel-class recognition, limiting their ability to capture broader range of distributional shifts.

Addressing both long-tailed recognition and broader distributional shifts, recent methods commonly incorporate auxiliary datasets and adapted outlier exposure mechanisms tailored for imbalanced settings.
PASCL \citep{wang2022partial} improves tail representation learning by modifying the SupCon \citep{khosla2020supervised} loss to operate selectively on tail classes and auxiliary OOD samples, encouraging better separation for underrepresented classes.
A different perspective is taken by COCL \citep{miao2024out}, which reformulates OOD detection as a classification task by introducing a dedicated outlier class and coupling it with logit adjustment to handle class imbalance.
Extending this formulation, EAT \citep{wei2024EAT} models OOD data using multiple outlier classes and leverages ensemble training to capture diverse outlier patterns.
In contrast, PATT \citep{he2024longtailedoutofdistributiondetectionprioritizing} focuses on modelling feature distributions, augmented with tail-aware attention and feature calibration to improve the separation between ID and OOD samples.

Despite recent progress, existing methods remain limited both in design and evaluation.
Methodologically, many approaches rely on simple linear classifiers that are misaligned with the complex geometry of learned feature spaces, while a single shared model must simultaneously capture dense head classes and sparse tail classes, leading to head-dominated representations and underfitting of tail regions.
Moreover, OOD detection is often implicitly coupled with classification, confusing true distributional uncertainty with data scarcity and resulting in ambiguous decision boundaries where tail samples resemble OOD.
From an evaluation standpoint, these methods are predominantly validated on natural image benchmarks such as CIFAR \citep{krizhevsky2009learning} and ImageNet \citep{ILSVRC15}, which differ significantly from medical imaging in structure and acquisition.
In addition, commonly used OOD benchmarks, including SVHN \citep{Netzer2011}, LSUN \citep{yu15lsun}, and INAT \citep{inat2024}, are relatively coarse-grained and fail to reflect the complexity and heterogeneity of real-world medical data distributions.

Motivated by these gaps, we propose a framework that generalises long-tailed OOD detection to complex, domain-specific distributions in medical imaging.
We introduce a novel classifier capable of modelling heterogeneous class structures in hyperspherical feature spaces.
To mitigate head-class-dominated representations, we further adopt a multi-expert framework that captures diverse label distribution regimes.
Finally, we decouple OOD detection from classification through a dedicated outlier expert, enabling explicit modelling of inlier–outlier separation.
To ground our approach, we first formalise the problem setup and notation.

\section{Preliminaries}
\label{sec_preliminaries}

\subsection{Problem Setup}
Let $\mathcal{X} \subset \mathbb{R}^v$ denote the image space of dimension $v$ and $\mathcal{Y}^{in} := {1, 2, \dots, K}$ be the number of ID classes, denoted by set $K$.
The training dataset $\mathcal{D}^{\text{train}} = \{(\mathbf{x}_i, y_i)\}_{i=1}^N$ consists of independent samples $(\mathbf{x}_i, y_i)$ drawn from the joint ID distribution $P_{\mathcal{X}\mathcal{Y}^{in}}$. Class probabilities follow a long-tailed prior, $P(Y=k) \propto k^{-\alpha}$ for some $\alpha > 0$.

In addition to the ID dataset, we assume access to an auxiliary outlier dataset $\mathcal{D}^{\text{aux}} = \{\mathbf{\tilde{x}_i} \}_{i=1}^N $, consisting of images that do not belong to the ID image space.
These samples are randomly collected and incorporated during training as additional targets in the loss function.
A detailed description of the auxiliary dataset is provided in Section \ref{sssec:aux_dataset_details}.

\textbf{Objective:} During inference, the model should perform two tasks: (1) binary OOD detection and (2) label prediction for samples deemed ID.
The OOD detection is performed by the function $S_{\text{OOD}}(\textbf{x})$,which assigns a scalar-value indicating how likely $\mathbf{x}$ is to be OOD.
A decision threshold $\lambda$ is then applied ot determine whether a sample is classified ID or OOD.
For samples classified as ID, the predicted label is determined by selecting the ID class with highest predicted probability value.

\subsection{von Mises-Fisher (vMF) distribution}
In long-tailed classification settings, learned feature norms often correlate with class frequency, causing classifiers that depend on vector magnitude to favour head classes \citep{menon2021longtail}.
A common mitigation strategy is to normalise embeddings onto a unit hypersphere, thereby enforcing emphasis on angular discrimination. 
In this normalised regime, the vMF distribution provides a well-suited probabilistic model for class-conditional densities defined on the hypersphere.

The vMF distribution is defined over the unit hypersphere $\mathbb{S}^{d-1} = \{\mathbf{x} \in \mathbb{R}^d : \| \mathbf{x} \|_2 = 1\}$,
where $d$ denotes the dimensionality of the embedding space produced by the encoder $f(\cdot)$. Its probability density function is given by
\begin{equation}
    p(\mathbf{x} \mid \boldsymbol{\mu}, \kappa)
    = C_d(\kappa)\exp\left(\kappa \boldsymbol{\mu}^{\top}\mathbf{x}\right),
\end{equation}
where $\boldsymbol{\mu} \in \mathbb{S}^{d-1}$ is the mean direction, $\kappa \geq 0$ is a concentration parameter controlling dispersion around $\boldsymbol{\mu}$, and $C_d(\kappa)$ is a normalization constant (which ensures that the density integrates to 1 over the hypersphere) given by,
\begin{equation}
    C_d(\kappa) =
    \frac{\kappa^{d/2-1}}{(2\pi)^{d/2} I_{d/2-1}(\kappa)}
\label{eq:vmf_norm}
\end{equation}
$C_d(\kappa)$ involves the modified Bessel function of the first kind $I_{\nu}(\cdot)$, where $\nu = d/2 -1$ is the order of the Bessel function, which arises from integrating the exponential term over the spherical manifold.
This dependence reflects the geometry of the hypersphere and distinguishes the vMF distribution from Gaussian densities defined in Euclidean space.

In prior work, vMF distributions have been restricted mainly to this formulation for constructing classifiers, where each ID class $c$ with a single mean direction $\boldsymbol{\mu}_c$ on the hypersphere and computing class posteriors via a softmax over the corresponding log-densities~\citep{ming2023cider, lu2024learning}.
Under this formulation, classification reduces to comparing the angular similarity between an input embedding and class prototypes, yielding the following posterior:
\begin{equation}
P(y=c \mid \mathbf{x}) =
\frac{\exp(\kappa \boldsymbol{\mu}_c^\top \mathbf{x})}
{\sum_{j=1}^K \exp(\kappa \boldsymbol{\mu}_j^\top \mathbf{x})}.
\end{equation}
When optimised via maximum likelihood, the resulting logit for class $c$ simplifies to a scaled cosine similarity, $\ell_c(\mathbf{x}) = \kappa \boldsymbol{\mu}_c^\top \mathbf{x}$.
As a consequence, conventional vMF-based classifiers are functionally equivalent to cosine classifiers, differing primarily in their probabilistic interpretation rather than their decision geometry.
While this formulation provides a principled directional model and naturally accommodates normalised embeddings, it also inherits the limitations of cosine classifiers: decision boundaries remain linear in angular space, and each class is represented by a single mode.
These assumptions can be overly restrictive in practice, particularly in long-tailed and open-set settings where classes exhibit heterogeneous intra-class variability and complex feature distributions.

Our work, on the other hand, builds upon this widely adopted vMF classifier formulation and extends it beyond this linear-boundary regime.
The proposed adaptation of the NvMF classifier is introduced in Section~\ref{subsec_vmf_classifier}, where we describe how richer class representations and improved OOD separation are achieved.

\section{Methodology}
\label{sec_methodology}

Our proposed framework \texttt{MARVEL}, the illustrated in figure \ref{fig:framework}, consists of two components: (1) a shared feature extractor $f : \mathcal{X} \mapsto \mathbb{R}^d$ that maps input $\mathbf{x} \in \mathcal{X}$ to high dimensional feature embedding $f(\mathbf{x})$; (2) an ensemble of experts $\mathcal{E}$. Each expert $g \in \mathcal{E}$ operates on the shared representation $f(\mathbf{x})$. 

The ensemble $\mathcal{E}$ is composed of two categories of experts. 
The first category consists of our proposed NvMF classifiers with margin awareness, denoted by $g^{NvMF}_{\tau}: \mathbb{R}^d \mapsto \mathbb{R}^{K+1}$, where $\tau \in \{0, 1, 2\}$ is the margin strength parameter (explained in Section~\ref{ssec:margin_adjusted_ensemble}), that specifies the subset of classes emphasised by the expert.
The second category is a dedicated outlier expert, denoted by $g^{\text{out}}: \mathbb{R}^d \mapsto \mathbb{R}^2$, implemented as a fully connected layer trained for binary classification.
Unlike the $K+1$ classifiers in the first category, this expert does not model fine-grained ID classes, but focuses solely on the ID/OOD decision.

During training, mini-batches are constructed by sampling from both $\mathcal{D}^{\text{train}}$ and $\mathcal{D}^{\text{aux}}$ equally.
For the NvMF experts, samples from $\mathcal{D}^{\text{train}}$ retain their original class labels $y \in \{1, \dots, K\}$, while samples from $\mathcal{D}^{\text{aux}}$ are assigned to the additional $(K+1)^{\text{th}}$ class.
For the outlier expert, samples from $\mathcal{D}^{\text{train}}$ are assigned the label $0$ (ID), irrespective of their class identity, whereas samples from $\mathcal{D}^{\text{aux}}$ are assigned the label $1$ (OOD).
Given these label assignments, each expert optimises its corresponding task-specific objective.
The NvMF classifiers are trained using losses $\mathcal{L}_{\text{NvMF}}^{\tau}$, while the outlier expert is trained with $\mathcal{L}{\text{out}}$.
The overall training objective jointly minimises the sum of these components:
\begin{equation}
\mathcal{L}_{\text{total}} = \sum_{\tau=\{0, 1, 2\}} \mathcal{L}^{\tau}_{NvMF} + \mathcal{L}_{out}
\end{equation}

\begin{figure*}
    \centering
    \includegraphics[width=\linewidth]{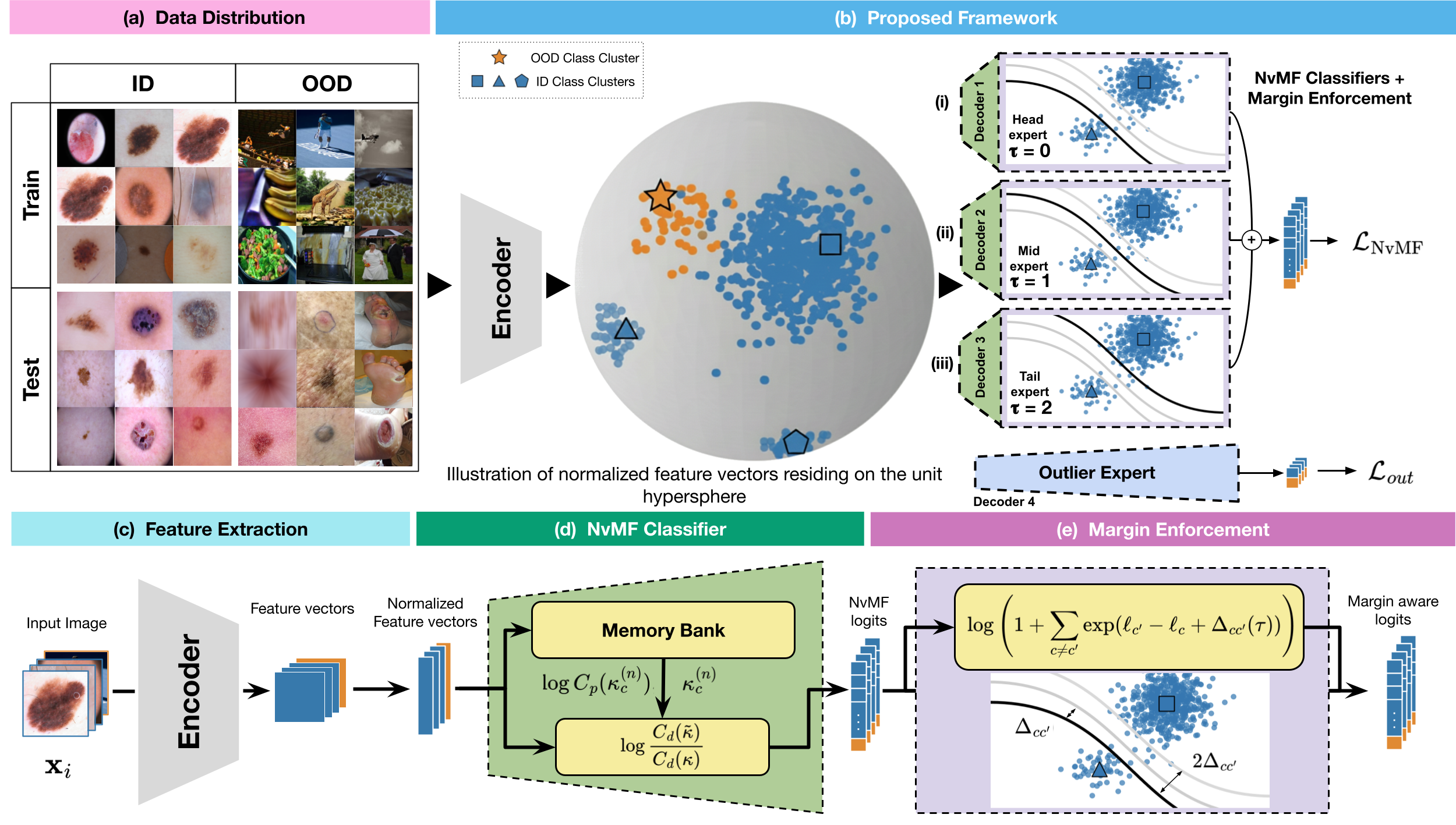}
    \caption{Overview of the proposed long-tailed OOD detection framework.
    (a) Schematic of the possible data distributions defining ID and OOD samples at training and test time.
    (b) High-level overview of the proposed framework: a shared encoder maps inputs to a normalised feature space, followed by specialised experts for head, mid, tail and outlier samples.
    (c) Feature extraction from the shared encoder, highlighting the inference path within the ID classification branch.
    (d) vMF classifier with a memory bank that maintains class-wise concentration parameters and normalization terms for logit computation.
    (e) Margin-aware mechanism controlled by parameter $\tau$, where different values of $\tau$ bias the decision boundary, enabling the classifier to specialise as a head-, mid-, or tail-class expert.}
    \label{fig:framework}
\end{figure*}

\subsection{Nonlinear von Mises-Fisher (NvMF) Classifier}
\label{subsec_vmf_classifier}

In \texttt{MARVEL}, the three experts in the ensemble are implemented as novel Nonlinear vMF classifiers (NvMF).
This design choice is motivated by the limitations of conventional hyperspherical classifiers. 
Cosine classifiers \citep{Gidaris_2018_CVPR} only encode the mean direction of each class and ignore the variability or concentration around it.
Prior work remedies this by modelling class-conditional densities using vMF distributions, as summarised in the preliminaries. 
However, because these models optimise the likelihood of a vMF density, the resulting logits reduce to scaled cosine similarity, meaning that their decision boundaries remain linear and inherit the same limitations as cosine classifiers.

Our approach takes a fundamentally different perspective: rather than treating the vMF distribution purely as a likelihood model, we exploit its exponential-family structure to construct class logits.
This reinterpretation yields a classifier, capable of inducing non-linear decision boundaries on the hypersphere, allowing it to better model heterogeneous cluster shapes and imbalanced embedding geometries. 
To make this connection precise, recall that the vMF distribution can be expressed as a member of the exponential family \citep[Section2.4]{10.5555/1162264}, with a sufficient statistic $\mathbf{T}(\mathbf{x}) = \mathbf{x}$, a natural parameter $\boldsymbol{\eta} = \kappa \boldsymbol{\mu}$, and a log-partition function $A(\boldsymbol{\eta}) = - \log C_d(\kappa)$. 

We model the distribution of normalised embeddings $\mathbf{x}$ as a mixture of class-conditional vMF components. Specifically, for a $(K+1)$-class problem:
\begin{equation}
    p(\mathbf{x}) = \sum_{c=1}^{K+1}w_c p(\mathbf{x}|c), \quad p(\mathbf{x}|c) = C_d(\kappa_c)\exp(\kappa_c\boldsymbol{\mu}_c^\top\mathbf{x}_c),
\end{equation}
where $w_c$ is the mixture weight, and \( (\boldsymbol{\mu}_c, \kappa_c) \) denote the mean direction and concentration of class \( c \), with \( \boldsymbol{\mu}_c \in \mathbb{S}^{d-1} \) and \( \kappa_c > 0 \). 
Each class-conditional component belongs to the exponential family and can be written in the canonical form, $p(\mathbf{x}|c) = \exp(\boldsymbol{\eta}_c^\top\mathbf{x} - A(\boldsymbol{\eta}_c))$. 
    
Since each class is modelled as a probability distribution $p(\mathbf{x}|c)$, the log-partition function $A(\boldsymbol{\eta}_c)$ serves as the necessary normalization term that scales the density to a valid probability. We can therefore interpret the classification task as evaluating how much this normalization constant must change to accommodate a new observation. If a point $\mathbf{x}$ aligns well with a class, the "cost" of including it in that distribution $p(\mathbf{x}|c)$ is low.
Thus, for any class $c$, we define the score $\ell_c= -\Delta A(\mathbf{x}) = A(\boldsymbol{\eta}_c) - A(\boldsymbol{\eta}_c + \mathbf{x})$, which measures how the log-partition function of the $c$-th component would change if $\mathbf{x}$ were attributed to that class.

Although the log-partition function acts as a global normalisation term, its class-conditional increments provide class-dependent compatibility scores on the hypersphere, inducing non-linear decision boundaries in the embedding space. We therefore use the collection of scores $\{l_c \}_{c=1}^{K+1}$ as logits in a discriminative classifier, defined as:
\begin{equation}
\ell_c = -\Delta A(\mathbf{x}) 
= - \log \frac{C_d(\|\kappa_c \boldsymbol{\mu}_c + \mathbf{x}\|)}{C_d(\kappa_c)}.
\label{eq:vmf_logit}
\end{equation}

In Theorem~\ref{th:logits}, we provide a theoretical justification for the proposed formulation by analysing the asymptotic regime of the vMF concentration parameter $\kappa$. Specifically, as $\kappa \to \infty$, the logits induced by our exponential-family construction converge to those used in existing vMF and cosine-based classifiers as proven in Theorem~\ref{th:logits}. 
\begin{theorem}
\label{th:logits}
Let $\mathbf{x},\boldsymbol{\mu}\in\mathbb{S}^{d-1}$. 
The vMF logit is defined as
\begin{equation}
\ell(\mathbf{x};\kappa,\boldsymbol{\mu}) 
= - \log \frac{C_{d}\!\big(\|\kappa\boldsymbol{\mu}+\mathbf{x}\|\big)}{C_{d}(\kappa)}.
\label{eqn:logit_form}
\end{equation}
As $\kappa \to \infty$, it converges to 
$
\boldsymbol{\mu}^\top\mathbf{x} + \mathcal{O}\!\big(\kappa^{-1}\big).
$
\end{theorem}

\begin{proof}
We analyse the asymptotic behaviour of the vMF logit as the concentration parameter $\kappa \to \infty$. The normalization constant $C_d(\kappa)$ depends on the modified Bessel function of the first kind $I_\nu(\kappa)$ with $\nu = d/2-1$ (refer to~\eqref{eq:vmf_norm}). Using its standard asymptotic expansion \citep{NIST:DLMF:10.40},

\begin{equation}
I_\nu(\kappa) \sim \frac{e^{\kappa}}{\sqrt{2\pi \kappa}}
\left(1 - \frac{4\nu^{2}-1}{8\kappa} + \mathcal{O}(\kappa^{-2})\right),
\quad \kappa \to \infty,
\label{eq:asymptomatic_expansion}
\end{equation}

Taking logarithms on both sides of~\eqref{eq:vmf_norm} and substituting the asymptotic expansion of $I_\nu(\kappa)$ from ~\eqref{eq:asymptomatic_expansion} yields

\begin{equation}
\log C_d(\kappa) = -\kappa + \tfrac{d-1}{2}\log\kappa + a_0 + \frac{a_1}{\kappa} 
+ \mathcal{O}(\kappa^{-2}),
\label{eq:log_C_d_kappa}
\end{equation}
for constants $a_0,a_1$ depending only on $d$. 

Next, letting $\rho = \boldsymbol{\mu}^\top\mathbf{x}$, we expand the term $\|\kappa \boldsymbol{\mu} + \mathbf{x}\|$ from~\eqref{eqn:logit_form}, 

\begin{equation}
\|\kappa \boldsymbol{\mu} + \mathbf{x}\|
= \sqrt{\kappa^2 + 2\kappa \rho + 1}
= \kappa + \rho + \frac{1-\rho^2}{2\kappa} + \mathcal{O}(\kappa^{-2})
\end{equation}

We now evaluate the asymptotic form of $\log C_d(r)$ at $r = \|\kappa \boldsymbol{\mu} + \mathbf{x}\|$ by substituting the above expansion in~\eqref{eq:log_C_d_kappa} and collecting terms up to order $\kappa^{-2}$, yielding

\begin{equation}
\log C_d(r) = -\kappa - \rho - \tfrac{1-\rho^{2}}{2\kappa} 
+ \tfrac{d-1}{2}\log\kappa + \tfrac{d-1}{2}\tfrac{\rho}{\kappa} 
+ a_0 + \mathcal{O}(\kappa^{-2})
\label{eq:log_C_d_r}
\end{equation}

Finally, to obtain~\eqref{eqn:logit_form}, we subtract~\eqref{eq:log_C_d_r} from~\eqref{eq:log_C_d_kappa}

\begin{equation}
\begin{split}
\ell(\mathbf{x};\kappa,\boldsymbol{\mu})
&= \log C_d(\kappa) - \log C_d(r) \\
&= \rho - \frac{1}{2\kappa}\big(-1 + \rho^2 + (d-1)\rho \big) 
  + \mathcal{O}(\kappa^{-2}).
\end{split}
\label{eqn:delta_A}
\end{equation}

Thus, in~\eqref{eqn:delta_A}, as $\kappa \to \infty$, the second term becomes 0, leaving the leading-order term $\rho$, which is $\boldsymbol{\mu}^\top\mathbf{x}$ (defined earlier), showing that the vMF logit converges to cosine similarity as $\kappa \to \infty$.
\end{proof}

This shows that the proposed vMF-based classifier recovers cosine-similarity–based classifiers as a limiting special case. Consequently, this formulation constitutes a principled generalisation of existing hyperspherical classifiers. While the large-
$\kappa$ regime yields linear decision boundaries on the sphere, smaller values of $\kappa$ induce non-linear decision boundaries.
The classifier is trained by applying the softmax over the logits $\ell_c$ and minimising the cross-entropy loss with respect to the ground truth class $y$:
\begin{equation}
    \mathcal{L}_{NvMF} = - \log \frac{\exp\ell_y}{\sum_{c=1}^{K+1} \exp \ell_c}
    = -\ell_y + \log \sum_{c=1}^{K+1} \exp \ell_c
\end{equation}

\subsection{Margin-aware ensemble of experts}
\label{ssec:margin_adjusted_ensemble}

With our proposed NvMF formulation, which can effectively capture non-linear decision boundaries, we aim to overcome the effect of class imbalance in data that tend to overfit to head classes while systematically underperforming on rare ones.
A key contributing factor is that decision boundaries are often placed too close to tail-class samples, making them vulnerable to misclassification.
One natural way to alleviate this issue is to enforce large inter-class margins, requiring the classifier to shift the decision boundaries towards head class and offer greater protection to tail class~\citep{cao2019learning, menon2021longtail}. 
We can enforce these margins by introducing pairwise margin objectives in the loss function, where a margin term $\Delta_{yc}$ modifies the pairwise logit difference between the ground-truth class $y$ and competing classes $c$. Thus, our loss function $\mathcal{L}_{NvMF}$ after introducing the objective becomes:
\begin{equation}
\begin{split}
\mathcal{L}_{NvMF}^{\Delta} 
&= -(\ell_y + \Delta_{yy}) + \log \sum_{c=1}^{K+1} \exp(\ell_{c} + \Delta_{yc})\\
&= \log \left( 1 + \sum_{c=1, c \neq y}^{K+1} \exp(\ell_{c} - \ell_{y} + \Delta_{yc} )\right)
\end{split}
\end{equation}

Following \citet{menon2021longtail}, these margins can be parametrised as 
$
\Delta_{yc} = \tau \log \pi_{c} / \pi_y
$
where $\pi_y,\pi_{c}$ are the empirical class priors computed from the training set, i.e., $\pi_y= n_y / \sum_{k} n_k$, where $n_y$ represents the number of training samples belonging to the class $y$, and $\tau \geq 0$ is the margin strength parameter that controls margin's intensity.
This formulation is inherently asymmetric, i.e., $\Delta_{yc} \neq \Delta_{cy}$.
For instance, if $y$ is a head class and $c$ is a tail class, then $\pi_{c} / \pi_{y} < 1$, implying $\log(\pi_{c} / \pi_{y}) < 0$, and thus $\Delta_{yc} < 0$ assuming $\tau > 0$. 
As a result, when comparing a head class against a tail class, $\Delta_{yc} < 0$, yielding a negative margin that penalises the head class.
Conversely, when comparing a tail class $y$ against a head class $c$, $\pi_{c} / \pi_{y} > 1$, implying $\log(\pi_{c} / \pi_y) > 0$ and hence $\Delta_{yc} > 0$, yielding a positive margin that increases the tail class logit.
Ultimately, this pairwise asymmetry shifts the decision boundaries towards under-represented class, thereby ameliorating the impact of imbalance.

The role of the margin strength parameter $\tau$ can be understood as follows.
When $\tau=0$, $\Delta_{cc'} = 0$, as illustrated in Fig.~\ref{fig:framework}b(i), effectively applying no margin at all. Thus, the classifier reduces to the default head-biased classifier.
As $\tau$ increases, $\Delta_{cc'}$ increases correspondingly, resulting in progressively stronger margins that favour tail classes (see Fig.~\ref{fig:framework}b(ii)–(iii)).
Thus, by incorporating $\Delta_{yc}=\tau \log (\pi_c/\pi_y)$, the resulting loss becomes:
\begin{equation}
\mathcal{L}_{NvMF}^{\tau} = \log \left( 1 + \sum_{c \neq y} \exp(\ell_{c}^\tau - \ell_{y}^\tau + \tau \log \frac{\pi_{c}}{\pi_y} )\right)
\end{equation}
where $\ell_c^\tau$ denotes the logit for class $c$, with superscript $\tau$ indexing the expert.

The proposed ensemble utilises $\tau \in \{0, 1, 2\}$ to integrate three specialised regimes: a head-biased expert ($\tau=0$) for standard accuracy, a balanced expert ($\tau=1$) for moderate rebalancing, and a tail-biased expert ($\tau=2$) for aggressive minority-class protection. By combining these complementary behaviours, \texttt{MARVEL} maintains high performance on majority classes while significantly improving robustness across mid- and tail-categories.

\subsection{Outlier Expert}
\label{subsec_outlier_expert}
While the NvMF experts model both ID classes and an additional OOD class, the resulting decision boundaries are inherently tied to multi-class discrimination.
As a result, OOD separation is influenced by inter-class competition among ID categories.
To provide a more direct supervision signal for OOD detection, we introduce a dedicated outlier expert $g^{\text{out}}: \mathbb{R}^d \to \mathbb{R}^2$, trained to perform binary ID versus OOD classification on the shared representation $f(\mathbf{x})$.
This expert complements the NvMF ensemble by learning global decision boundary that explicitly separates ID and OOD samples.

Given the balanced presence of ID and OOD samples in each training batch, the binary classification task is well-conditioned and does not require complex decision boundaries.
Accordingly, the outlier expert is implemented as a simple fully connected layer \citep{alexnet}.
A standard cross-entropy loss is used for training over, where $\ell_c^{\text{ood}}$ denotes the logit corresponding to class $c \in \{0,1\}$, and $y \in \{0, 1\}$ is the ground truth label.
\begin{equation} 
\mathcal{L}_{\text{OOD}} = -\ell^{\text{ood}}_{y} + \log \left( \sum_{c\in\{0,1\}}\exp \ell^{\text{ood}}_{c} \right)
\end{equation}

\subsection{Inference}
\label{sec:inference}
Given an input sample, a forward pass through the proposed model produces the following logits.
The NvMF experts output a set of three logit vectors $\{\boldsymbol{\ell}^\tau\}_{\tau=0}^{2}$, where each $\boldsymbol{\ell}^\tau \in \mathbb{R}^{K+1}$ corresponds to the logits for $K+1$ classes.
In addition, the outlier expert produces $\boldsymbol{\ell}^{\text{ood}} \in \mathbb{R}^2$, representing the ID, OOD classes.
These logits are used differently for ID classification and OOD detection, as detailed in the following subsections.

\subsubsection{ID Classification}
\label{ssec:id_classification}
To perform ID classification, we utilise the predictions from the NvMF experts only.
Although each expert produces predictions over $K+1$ classes, we restrict our attention to the first $K$ logits.
For each expert $\tau \in \{0,1,2\}$, we compute the class probabilities over the ID classes as,
$
p_c^{\tau} = \mathrm{softmax}(\boldsymbol{\ell}^\tau_{1:K})c.
$
We then aggregate the predictions across all experts to obtain the final class probabilities:
\begin{equation}
p{c} =
\frac{1}{3} \sum_{\tau=0}^{2} p^{\tau}_{c}.
\end{equation}
The predicted class label is given by $\arg\max_c , p_c$.
This aggregation leverages the complementary strengths of the experts, each capturing different aspects of the label distribution.

\subsubsection{OOD Detection}
\label{subsec_ood_detection}

To perform OOD detection, we leverage two complementary signals: (i) the probability of a sample being assigned to $(K+1)$-th class by NvMF experts, and (ii) the OOD probability predicted by the dedicated Outlier Expert.

For each NvMF expert $\tau \in \{0,1,2\}$, we first compute the probability assigned to the $(K+1)$-th class as,
$
s_{\text{NvMF}}^{\tau} = \mathrm{softmax}(\boldsymbol{\ell}^\tau)_{K+1}.
$
We then aggregate the predictions across all experts to obtain a unified NvMF-based score:
\begin{equation}
s_{\text{NvMF}} = \frac{1}{3} \sum_{\tau=0}^{2} s_{\text{NvMF}}^{\tau}.
\end{equation}
In parallel, the Outlier Expert produces an explicit estimate of OOD likelihood, given by:
$
s_{\text{ood}} = \mathrm{softmax}(\boldsymbol{\ell}^{\text{ood}})_{1}.
$
Finally, we combine these two signals to obtain:
\begin{equation}
S_{\text{OOD}} = \frac{1}{2} \left( s_{\text{NvMF}} + s_{\text{ood}} \right).
\end{equation}
where, higher values of $S_{\text{OOD}}(\mathbf{x})$ indicate greater likelihood of the sample being OOD.
This formulation leverages the complementary strengths of (i) class-aware NvMF classifiers that contextualise OODness relative to ID decision boundaries, and (ii) a binary expert specialised in learning coarse ID/OOD separation. Together, they yield robust OOD scoring under both balanced and long-tailed settings.

\section{Experimental Setup}
\label{sec: Experimental Setup}

\subsection{Dataset details}
\label{ssec:dataset_details}

To comprehensively evaluate the proposed framework under realistic and challenging conditions, we adopt an extensive evaluation protocol that combines long-tailed ID datasets with carefully curated OOD datasets. For each ID dataset, we consider multiple OOD datasets that span a spectrum of semantic difficulty, ranging from nearOOD (hard) to farOOD (easy), enabling a fine-grained analysis of OOD detection performance. Figure~\ref{fig:ood-datasets-spectrum} provides an overview of the dataset composition and evaluation setting, illustrating representative ID samples, their corresponding near- and farOOD counterparts, as well as the long-tailed label distributions of the ID datasets.  

\subsubsection{ID Dataset Details}
\label{sssc:id_dataset_details}
We evaluate the proposed method on three publicly available medical imaging datasets: (i) the Retinal Fundus Multi-Disease Image Dataset (RFMiD)\footnote{\url{https://riadd.grand-challenge.org/download-all-classes/}} \citep{rfmid-dataset}, (ii) the International Skin Imaging Collaboration 2019 (ISIC2019) dataset\footnote{\url{https://challenge.isic-archive.com/data}} \citep{codella2017isic} and (iii) NCT Colorectal Caner (NCTCRC) dataset\footnote{\url{https://zenodo.org/records/1214456}} \citep{kather_2018_100k}.

ISIC2019 comprises dermatoscopic images of pigmented skin lesions spanning eight diagnostic categories, and was curated primarily to support automated skin lesion analysis and melanoma classification. 
NCTCRC is a histopathological image dataset for nine-class tissue classification, created to facilitate the development of computational models for colorectal cancer tissue characterization from hematoxylin and eosin (H\&E) stained whole-slide images. 
Finally, RFMiD consists of retinal fundus images annotated by experts for 46 ocular conditions, and was introduced to benchmark automated screening and diagnosis of a broad spectrum of retinal diseases, including rare pathologies, thereby forming a challenging long-tailed benchmark. 
For RFMiD, to conform to a strictly multi-class single-label classification setting (the scope of this study), we retain only samples with a single annotation and discard those with multiple labels.
To enable systematic evaluation of OOD detection, each dataset is partitioned at the class level, with a subset of semantic categories withheld from training and treated as OOD at test time. 
For all datasets, the train/test split and the ID/OOD classes are summarised in Table~\ref{tab:dataset_summary}. We use the standard dataset splits defined in the original benchmarks and adopted in prior work, ensuring that no data leakage occurs between the training and evaluation sets.

To reflect the severe class imbalance encountered in real-world medical imaging scenarios, we intentionally reshape the class distribution of the training data to follow a long-tailed regime.
Concretely, we model this imbalance using a Pareto distribution \citep{arnold2015pareto}, which provides a principled and widely used approximation of long-tailed class frequencies. 
We set the imbalance ratio $\rho = 50$, corresponding to a $50{:}1$ ratio between the sample counts of the most frequent (head) and least frequent (tail) classes. 
The resulting label distributions for all datasets are illustrated in Fig.~\ref{fig:ood-datasets-spectrum}(b).

\begin{center}
\captionof{table}{Summary of datasets used in this study. (A) Dataset statistics and OOD evaluation splits. (B) ID and novel OOD class labels.}
\label{tab:dataset_summary}

\resizebox{\linewidth}{!}{
\begin{tabular}{l cccccc}
\toprule
 & \multicolumn{6}{c}{\textbf{Summary of Datasets}} \\
\cmidrule(lr){2-7}
Dataset & Train & Val & Test & \#Novel & NearOOD1 & NearOOD2 \\
\midrule
\multicolumn{7}{l}{\textbf{A. Dataset splits \& statistics}} \\
\quad RFMiD    & 846 & 368 & 2810 & 113 & 256 & 968 \\
\quad ISIC2019 & 6585 & 3531 & 2396 & 3822 & 2657 & 2298 \\
\quad NCT-CRC  & 15555 & 7262 & 20748 & 9252 & 1695 & 49260 \\
\midrule
\multicolumn{7}{l}{\textbf{B. ID/OOD class labels}} \\

\quad RFMiD    & \multicolumn{6}{p{\linewidth}}{%
\textbf{ID:} Macular Hole (MH), Diabetic Retinopathy (DR), Optic Disc Cupping (ODC), 
Drusen (DN), Age-related Macular Degeneration (ARMD), Branch Retinal Vein Occlusion (BRVO), 
Optic Disc Edema (ODE), Myopia (MYA), Retinitis (RS), Central Serous Retinopathy (CSR), 
Central Retinal Vein Occlusion (CRVO), Optic Disc Pit (ODP), Retinal Tear (RT), 
Epiretinal Membrane (EDN), Macular Hole with Lamellar features (MHL), Macular Scar (MS)

\textbf{OOD:} Tessellated Fundus (TSLN), Lens Subluxation (LS), Strabismus (ST), 
Retinitis Pigmentosa (RP), Cotton Wool Spots (CWS), Choroidal Bands (CB), 
Optic Disc Pit Maculopathy (ODPM), Preretinal Hemorrhage (PRH), Myelinated Nerve Fibers (MNF), 
Hypertensive Retinopathy (HR), Central Retinal Artery Occlusion (CRAO), 
Tilted Disc (TD), Cystoid Macular Edema (CME), Post-traumatic Chorioretinopathy (PTCR), 
Chorioretinal Fold (CF), Vitreous Hemorrhage (VH), Macroaneurysm (MCA), 
Vascular Sheathing (VS), Branch Retinal Artery Occlusion (BRAO), 
Plaque (PLQ), Hemorrhagic PED (HPED), Choroiditis Lesion (CL), 
Tortuous Vessels (TV), Papilledema (PT)} \\

\quad ISIC2019 & \multicolumn{6}{p{\linewidth}}{%
\textbf{ID:} Melanoma (MEL), Basal Cell Carcinoma (BCC), Benign Keratosis (BKL), 
Actinic Keratosis (AK), Dermatofibroma (DF), Vascular Lesion (VASC)

\textbf{OOD:} Melanocytic Nevus (NV), Melanoma (MEL)} \\

\quad NCT-CRC  & \multicolumn{6}{p{\linewidth}}{%
\textbf{ID:} Tumor Epithelium (TUM), Muscle (MUS), Lymphocytes (LYM), 
Stroma (STR), Adipose Tissue (ADI), Mucus (MUC)

\textbf{OOD:} Background (BACK), Debris (DEB), Normal Colon Mucosa (NORM)} \\
\bottomrule
\end{tabular}
}
\end{center}

\subsubsection{OOD Dataset Details}
\label{sssec:ood_dataset_details}
For OOD detection, we evaluate along a spectrum of OOD scenarios that range from challenging nearOOD cases to relatively easier farOOD cases as described below:

At the most challenging end of the spectrum are \emph{novel-classes}, where test samples originate from the same underlying data distribution as the training set but belong to semantic classes that are entirely unseen during training. To simulate this scenario, we withhold a subset of classes from the training data and treat samples from these classes as OOD at evaluation time.

The next level considers \emph{corruption-based} OOD settings, where test images are generated by applying synthetic acquisition and processing degradations to in-distribution samples. Specifically, we simulate common real-world corruptions using the Albumentations library \citep{2018arXiv180906839B}, including Gaussian ($\sigma \in$ [0.2, 0.3]) and ISO noise (color shift $\in$ [0.1, 0.5], intensity $\in$ [1.0, 2.0]), motion (kernel size $\leq 121$ pixels) and zoom blur (magnification factor $\leq 5$), illumination artifacts through Random Sun Flare (source radius $=500$ pixels), compression artifacts (quality $\in$ [1, 10]), resolution degradation (scale $\in$ [0.02, 0.10]), and structural perturbations such as pixel (pixel dropout $=50\%$) and grid dropout (grid ratio $=0.75$). All corruption procedures are implemented in a reproducible script released as part of our codebase.\footnote{Provided in \url{https://github.com/redboxup/MARVEL}.}
    
At the third level, we consider \textit{distribution-shifted} OOD settings, where OOD samples arise from changes in imaging modality, acquisition setup, or anatomical focus. For each ID dataset, we select external datasets that induce progressively stronger distribution shifts, denoted as NearOOD1 and NearOOD2: 
\begin{enumerate}
    \item For ISIC2019, PAD-UFES \citep{PACHECO2020106221} comprises skin lesion images captured using smartphone cameras, while DFU \citep{Cassidy2021DFUC} consists of images of diabetic foot ulcers. We treat PAD-UFES as NearOOD1, representing a moderate distribution shift due to the change in acquisition device while retaining similar skin cancer pathologies, whereas DFU is designated as NearOOD2, reflecting a more substantial shift that involves both a different anatomical region and distinct semantic categories.
    \item For NCTCRC, the MIHIC \citep{Wang2024MIHIC} dataset comprises multiplex immunohistochemically (IHC) stained histological image patches from lung cancer tissue. In contrast, the cervical cytology dataset \citep{phoulady2018newcervicalcytologydataset} consists of Pap smear slide images for precancerous and cancerous stages. We treat MIHIC as NearOOD1, representing a moderate distribution shift due to the use of IHC staining as opposed to standard H\&E staining while retaining histopathological tissue structure. The Pap cytology dataset is designated as NearOOD2, reflecting a more substantial shift both in tissue type and imaging modality.
    \item For RFMiD, DeepDRiD \citep{LIU2022100512} comprises ultra-wide-field fundus images, while OCT \citep{Kermany2018Cell} consists of cross-sectional optical coherence tomography images of the retina. We treat DeepDRiD as NearOOD1, representing a moderate distribution shift due to the change in imaging modality, whereas OCT is designated as NearOOD2, reflecting a more substantial shift arising from a fundamentally different imaging technique.
\end{enumerate}
    
Finally, at the fourth level, we consider \emph{farOOD} samples drawn from the MS-COCO-5k \citep{Lin2014COCO} dataset, which comprises natural images of everyday objects and scenes. These images are visually and semantically distant from the medical domains considered in our study, making MS-COCO-5k a challenging farOOD benchmark.

Together with corruptions, domain-shifted and farOOD settings, these configurations establish a structured and graded OOD spectrum for each dataset, enabling systematic benchmarking of robustness and generalisation of the proposed method across medical imaging domains.

\begin{figure*}
    \centering
    \includegraphics[width=0.85\linewidth]{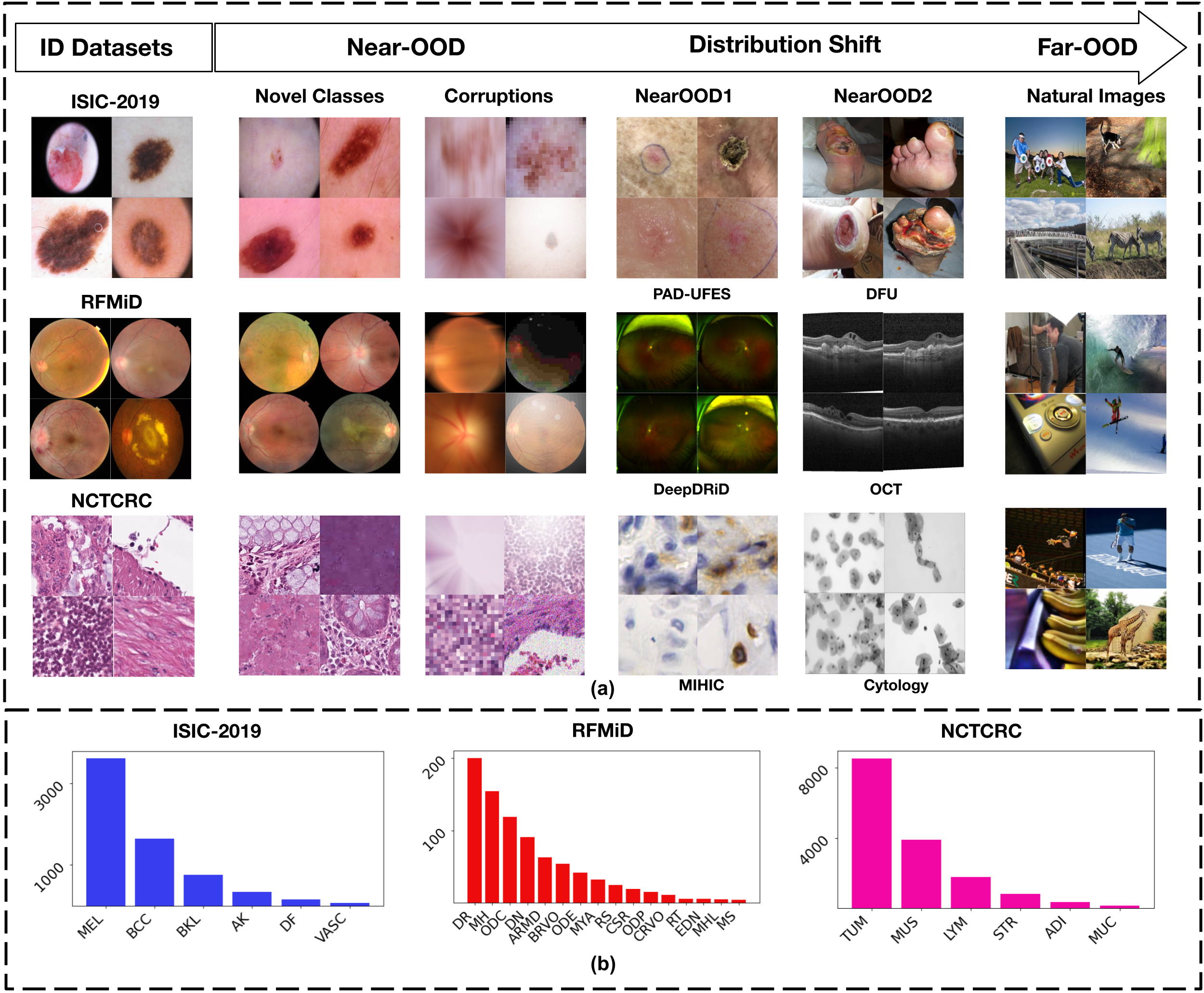}
    \caption{Dataset composition and evaluation setup for long-tailed OOD detection. (a) Shows representative samples from the ID datasets together with their corresponding nearOOD and farOOD evaluation data. (b) Presents the class label distributions of the ID datasets.}
    \label{fig:ood-datasets-spectrum}
\end{figure*}

\subsection{Implementation Details}  
\label{ssec:imp_details}
All models are implemented in PyTorch and trained on NVIDIA A6000 GPUs using a ResNet-18 backbone pretrained on ImageNet. 
We use the Adam optimizer with a cosine learning-rate schedule (no warmup) and train all models for 75 epochs with an initial learning rate of $1\times10^{-4}$. All reported results are averaged over seven independent random seeds.
For the state-of-the-art comparison, for other methods, we keep the optimizer, number of epochs, learning rate, and batch size identical to our training protocol whenever applicable. For these methods, hyperparameters are adopted as specified in the original publications (or tuned to achieve better performance for a fair comparison).
\subsubsection{Auxiliary Dataset for OOD detection training}
\label{sssec:aux_dataset_details}
We use ImageNet-100 \citep{ILSVRC15}, a curated subset of the ImageNet dataset containing 100 classes sampled from the original 1,000-class hierarchy. The dataset encompasses diverse natural images spanning a wide range of objects and scenes, providing a broad and semantically distinct set of samples that do not overlap with the medical ID datasets. Its balanced class composition and high variability make it suitable for training the outlier expert to recognise generic OOD features without biasing ID classification.

\subsection{Evaluation Metrics} 
\label{ssec:eval_metrics}

To assess the performance of models for ID classification, we report both the standard accuracy and balanced accuracy. Standard accuracy is defined as the fraction of correctly classified samples,
\begin{equation}
\text{Acc} = \frac{1}{N}\sum_{c=1}^{K} \text{TP}_c
\end{equation}
where $K$ is the number of classes and $N$ is the total number of samples. While standard accuracy measures the overall proportion of correctly classified samples, it can be biased towards the majority (head) classes in long-tailed datasets. Balanced accuracy mitigates this bias by averaging per-class recall,
\begin{equation}
\text{BAcc} = \frac{1}{K} \sum_{c=1}^{K} \frac{\text{TP}_c}{\text{TP}_c + \text{FN}_c}
\end{equation}
Additionally, we report accuracy for head, mid, and tail classes individually to provide a detailed view of model behaviour across the label distribution.

For OOD detection, we evaluate the model's ability to distinguish ID from OOD samples using three standard uncertainty-based metrics: (i) the Area Under the Receiver Operating Characteristic curve (AUROC) quantifies the model's discrimination capability across all possible thresholds, reflecting how well ID and OOD samples can be separated, (ii) the Area Under the Precision-Recall curve (AUPR) complements AUROC by focusing on the trade-off between precision and recall, which is particularly informative for imbalanced OOD scenarios, and (iii) the False Positive Rate at 95\% True Positive Rate (FPR95) measures the rate at which OOD samples are incorrectly classified as ID when the model correctly identifies 95\% of the ID samples, providing an interpretable measure of OOD misclassification under high-confidence conditions. For statistical comparison, we perform a paired two-tailed $t$-test across the $7$ runs between our method and the strongest baseline using identical seeds, ensuring paired observations.

\subsection{Experiments}
\label{ssec:exps}

\subsubsection{Comparison of the proposed OOD detection with state-of-the-art methods}
\label{sssec:sota_comp_ood}
We compare the OOD detection performance of the proposed \texttt{MARVEL} framework with state-of-the-art methods.
To reflect current research trends, we consider long-tailed recognition methods proposed in recent years, focusing primarily on work publised since 2022.
Earlier methods such as OLTR and HODL are not included, as more recent approaches offer improved formulations for handling class imbalance and are more representative of the present state of the field.
The selected baselines consist of PASCL, COCL, EAT, and PATT, which are widely used recent methods that address OOD detection under long-tailed distributions and capture diverse modelling strategies in this setting.
In addition, we include OE as a reference baseline that leverages auxiliary outlier data to enhance OOD detection performance, although it is not specifically designed to handle long-tailed data.
The evaluation considers multiple OOD scenarios that span a spectrum of distributional shifts, ranging from FarOOD samples to more challenging nearOOD settings, as described in Section~\ref{sssec:ood_dataset_details} using the metrics described in Section~\ref{ssec:eval_metrics}. This experimental design enables a systematic assessment of each method’s ability to distinguish ID samples from a broad range of OOD data encountered in medical settings.

\subsubsection{Comparison of the proposed ID classification with state-of-the-art methods}
\label{sssec:sota_comp_id}
We compare the ID classification performance of \texttt{MARVEL} framework against those of state-of-the-art methods. We consider the evaluation metrics specificed in Section~\ref{ssec:eval_metrics}, including balanced accuracy and individual accuracies for head, mid, and tail classes to characterise the consistency of classification performance across different class-frequency regimes.

\subsubsection{Ablation study I: Impact of the choice of the classifier}
\label{sssec:abl1_class_form}
We investigate the effect of classifier design on both ID classification and OOD detection within the \texttt{MARVEL} framework.
We compare our NvMF classifier, capable of learning non-linear decision boundaries, against commonly used baseline classifiers, including fully connected (FC), cosine and simple vMF classifiers. 
For OOD prediction, we consider the same set of classifier variants and additionally evaluate a setting in which no dedicated outlier expert is employed. 

\subsubsection{Ablation study II: Effect of the number of experts within the ensemble}
\label{sssec:abl2_exp_arch}
In this experiment, we investigate the effect of varying the number of experts within the ensemble and assess the contribution of the dedicated outlier expert. We construct models with one, two, three, and four experts to study how ensemble size influences both ID classification and OOD detection performance. 

\subsubsection{Confidence calibration analysis}
\label{sssec:calibration_experts}
In addition to accuracy and detection performance, we evaluate the confidence calibration of different methods, which is critical for reliable deployment in medical OOD detection settings. Well-calibrated models produce confidence scores that accurately reflect their empirical correctness, enabling more trustworthy decision-making under uncertainty.

\subsubsection{Risk-Coverage analysis}
\label{sssec:risk_coverage}
To evaluate model prediction reliability under uncertainty, we analyse the risk-coverage trade-off for RFMiD, ISIC2019, and NCTCRC datasets for OOD detection. Risk-coverage curves show how prediction error (\textit{risk}) varies with \textit{coverage}, which is the retained ID sample proportion, obtained by rejecting low-confidence predictions. To construct these curves, the ID and OOD samples are equally sampled (to match the size of the OOD samples or the split size of the ID test, whichever is lower, in Table~\ref{tab:dataset_summary}), so the 0.5 coverage considers all ID samples for a robust model. Here, \emph{coverage} is the fraction of samples for which the model makes a prediction (ID classification), while \emph{risk} is the error rate on these retained samples. A risk of zero up to 0.5 coverage indicates perfect separation, with all ID samples correctly retained and OOD samples rejected. For lower coverage, as the model discards low-confidence predictions and retains reliable ones, risk is lower. As coverage exceeds 0.5, OOD samples are included, thus increasing the risk. This analysis shows the model's ability to maintain low error rates under high-confidence regimes, crucial in safety-critical medical applications.

\subsubsection{Effect of the choice of the OOD detector}
\label{sssec:ood_det}
Finally, this study evaluates the framework under a range of OOD detection heads and scoring mechanisms, including logit-based (MSP, MLS), energy-based (EBO), distance-based detectors (KNN, MD) and hybrid approaches (NN-Guide). This evaluation examines the compatibility of the framework with different detection criteria and provides a basis for selecting an appropriate scoring mechanism in settings with distinct operational requirements. 

\subsubsection{Qualitative analysis of learned representations}
\label{sssec:tsne}
To gain further insight into the learned representation by the proposed framework, we perform a qualitative analysis using t-SNE visualizations of the learned feature space on the NCTCRC dataset (since it comprises sufficient samples in the tail classes to appreciate the separability). 
We compare three settings: (a) a baseline model trained with a vMF classifier, (b) the proposed NvMF classifier, and (c) the margin-aware NvMF multi-expert framework. 
The embeddings are obtained from the penultimate feature layer and projected into two dimensions for visualization.
This analysis provides an intuitive understanding of how the margin-aware experts contribute to improved discrimination for ID classification.

\section{Results}

\subsection{Comparison of the proposed OOD detection with state-of-the-art methods}
Tables \ref{tab:results_rfmid}, \ref{tab:results_isic} and \ref{tab:results_nctcrc} summarise the results of comparison of \texttt{\texttt{MARVEL}} with the recent state-of-the-art methods. 
Across all three datasets, the proposed \texttt{MARVEL} consistently achieves the strongest overall performance. In particular, \texttt{MARVEL} obtains the best average AUROC, AUPR and FPR95 on RFMiD, ISIC2019, and NCTCRC, significantly higher than other baselines (with p-value $<$ 0.01) demonstrating that the proposed framework generalises well across different medical imaging domains and OOD scenarios.
Among the competing methods, COCL emerges as the second-best baseline on RFMiD and ISIC2019, while EAT achieves the second-best performance on NCTCRC.
Notably, \texttt{MARVEL} performs perfectly on the FarOOD setting across all datasets (except NCTCRC, still achieving AUROC $>$ 99.00), achieving near-saturated AUROC and AUPR values with essentially zero FPR95.
This strong performance can be attributed to the dedicated inlier–outlier expert detector, which strongly separates farOOD samples from the ID data.
In contrast, open-set detection remains the most challenging scenario across methods.
While \texttt{MARVEL} still achieves the best results in this setting, the AUROC values remain comparatively lower with performance consistently lower than 85\% AUROC, whereas the remaining OOD categories (NearOOD and Corruptions) largely surpass this threshold.

From Table~\ref{tab:results_rfmid}, 
\texttt{MARVEL} achieves the best overall performance with a significantly higher average AUROC, AUPR and FPR95, outperforming all baseline methods. COCL provides the second-best average AUROC of 87.38, largely driven by its near-saturated results on NearOOD2 (99.99 AUROC) and FarOOD (100.00 AUROC), which allows it to remain relatively close to \texttt{MARVEL} in terms of average performance.
However, COCL performs noticeably worse on other parts of the OOD spectrum, particularly in the NearOOD1 and Corruption settings where it achieves significantly less (p-value $<$ 0.01) AUROC with a differences of 7.07\% and 7.27\% respectively, compared to \texttt{MARVEL}. For open-set detection, while \texttt{MARVEL} still achieves the best AUROC, the second-best result is obtained by EAT.

\begin{table*}[t]
\centering
\scriptsize
\setlength{\tabcolsep}{1pt}
\caption{Comparison of OOD detection results of our method with those of state-of-the-art methods on the RFMiD dataset (mean $\pm$ standard deviation). The compared methods include OE \citep{hendrycks2019oe}, PASCL \citep{wang2022partial}, EAT \citep{wei2024EAT}, PATT \citep{he2024longtailedoutofdistributiondetectionprioritizing}, COCL \citep{miao2024out}. Performance is measured using AUROC, AUPR, and FPR95, higher AUROC and AUPR, and lower FPR95 indicate superior performance. The best results are highlighted in \textbf{bold}, while the second-best results are \uline{underlined}. $[\ddagger]$ $p < 0.01$, $[\dagger]$ $p < 0.05$; paired $t$-test with respect to the top baseline result.}
\begin{adjustbox}{max width=\textwidth}
\begin{tabular}{l|ccc|ccc|ccc|ccc|ccc|ccc}
\toprule
\multirow{2}{*}{\textbf{Method}}
& \multicolumn{18}{c}{\textbf{RFMiD}} \\
\cmidrule(lr){2-19}
& \multicolumn{3}{c}{\textbf{Open Set}}
& \multicolumn{3}{c}{\textbf{NearOOD1}}
& \multicolumn{3}{c}{\textbf{NearOOD2}}
& \multicolumn{3}{c}{\textbf{Corruptions}}
& \multicolumn{3}{c}{\textbf{FarOOD}}
& \multicolumn{3}{c}{\textbf{Average}} \\
\cmidrule(lr){2-4}\cmidrule(lr){5-7}\cmidrule(lr){8-10}
\cmidrule(lr){11-13}\cmidrule(lr){14-16}\cmidrule(lr){17-19}
& AUROC & AUPR & FPR95
& AUROC & AUPR & FPR95
& AUROC & AUPR & FPR95
& AUROC & AUPR & FPR95
& AUROC & AUPR & FPR95
& AUROC & AUPR & FPR95 \\
\midrule

OE
& 53.60 & 79.33 & 95.28 & 49.82 & 66.95 & 100.00 & 88.97 & 86.83 & 66.91 & 68.77 & 53.81 & 71.93 & 92.67 & 57.07 & 34.90 & 70.77 & 68.80 & 73.80 \\
& {\scriptsize$\pm$1.15} & {\scriptsize$\pm$0.32} & {\scriptsize$\pm$1.35}
& {\scriptsize$\pm$14.05} & {\scriptsize$\pm$13.77} & {\scriptsize$\pm$0.00}
& {\scriptsize$\pm$7.26} & {\scriptsize$\pm$8.61} & {\scriptsize$\pm$33.09}
& {\scriptsize$\pm$10.87} & {\scriptsize$\pm$13.75} & {\scriptsize$\pm$19.10}
& {\scriptsize$\pm$0.60} & {\scriptsize$\pm$8.30} & {\scriptsize$\pm$3.88}
& {\scriptsize$\pm$3.25} & {\scriptsize$\pm$3.66} & {\scriptsize$\pm$9.23} \\

PASCL
& 54.83 & 80.52 & 93.81 & 80.90 & 87.72 & 88.02 & 94.30 & 91.68 & 31.78 & 66.80 & 49.84 & 76.46 & 98.60 & 95.82 & 2.37 & 79.09 & 81.12 & 58.49 \\
& {\scriptsize$\pm$2.37} & {\scriptsize$\pm$1.85} & {\scriptsize$\pm$0.88}
& {\scriptsize$\pm$5.96} & {\scriptsize$\pm$6.04} & {\scriptsize$\pm$5.92}
& {\scriptsize$\pm$2.49} & {\scriptsize$\pm$5.23} & {\scriptsize$\pm$18.33}
& {\scriptsize$\pm$2.70} & {\scriptsize$\pm$2.15} & {\scriptsize$\pm$14.46}
& {\scriptsize$\pm$1.83} & {\scriptsize$\pm$1.52} & {\scriptsize$\pm$2.16}
& {\scriptsize$\pm$0.87} & {\scriptsize$\pm$0.93} & {\scriptsize$\pm$0.94} \\

EAT
& \uline{58.67} & \uline{80.95} & \uline{91.15} & 82.99 & 90.01 & 80.60 & 74.54 & 76.37 & 98.38 & 76.98 & 56.23 & 47.24 & 67.77 & 57.95 & 99.13 & 72.19 & 72.30 & 83.30 \\
& {\scriptsize$\pm$0.54} & {\scriptsize$\pm$1.26} & {\scriptsize$\pm$1.53}
& {\scriptsize$\pm$5.59} & {\scriptsize$\pm$3.36} & {\scriptsize$\pm$14.63}
& {\scriptsize$\pm$9.63} & {\scriptsize$\pm$9.14} & {\scriptsize$\pm$1.71}
& {\scriptsize$\pm$3.60} & {\scriptsize$\pm$4.46} & {\scriptsize$\pm$14.43}
& {\scriptsize$\pm$9.59} & {\scriptsize$\pm$12.74} & {\scriptsize$\pm$1.38}
& {\scriptsize$\pm$4.46} & {\scriptsize$\pm$4.38} & {\scriptsize$\pm$2.34} \\

PATT 
& 50.05 & 77.56 & 92.04 & 80.59 & 88.13 & 83.85 & 80.12 & 79.43 & 63.50 & 71.95 & 67.62 & 79.17 & 95.35 & 89.15 & 31.97 & 75.61 & 80.38 & 70.10 \\
& {\scriptsize$\pm$4.52} & {\scriptsize$\pm$1.86} & {\scriptsize$\pm$2.65}
& {\scriptsize$\pm$7.83} & {\scriptsize$\pm$4.80} & {\scriptsize$\pm$11.17}
& {\scriptsize$\pm$23.12} & {\scriptsize$\pm$24.10} & {\scriptsize$\pm$54.44}
& {\scriptsize$\pm$16.43} & {\scriptsize$\pm$18.52} & {\scriptsize$\pm$23.98}
& {\scriptsize$\pm$7.24} & {\scriptsize$\pm$13.23} & {\scriptsize$\pm$51.36}
& {\scriptsize$\pm$1.51} & {\scriptsize$\pm$3.10} & {\scriptsize$\pm$12.39} \\

COCL
& 53.23 & 80.17 & 94.40 & \uline{96.16} & \uline{97.63} & \uline{22.27} & \uline{99.99} & \uline{99.99} & \textbf{0.00} & \uline{87.51} & \uline{74.95} & \uline{31.15} & \textbf{100.00} & \uline{99.98} & \textbf{0.00} & \uline{87.38} & \uline{90.54} & \uline{29.56} \\
& {\scriptsize$\pm$4.88} & {\scriptsize$\pm$3.29} & {\scriptsize$\pm$2.23}
& {\scriptsize$\pm$5.22} & {\scriptsize$\pm$3.22} & {\scriptsize$\pm$33.15}
& {\scriptsize$\pm$0.01} & {\scriptsize$\pm$0.02} & {\scriptsize$\pm$0.00}
& {\scriptsize$\pm$2.96} & {\scriptsize$\pm$6.55} & {\scriptsize$\pm$4.46}
& {\scriptsize$\pm$0.00} & {\scriptsize$\pm$0.03} & {\scriptsize$\pm$0.00}
& {\scriptsize$\pm$0.42} & {\scriptsize$\pm$0.92} & {\scriptsize$\pm$6.30} \\

MARVEL
& \textbf{59.85} & \textbf{81.55} & \textbf{91.15} & \textbf{99.85}\textsuperscript{$\ddagger$} & \textbf{99.91}\textsuperscript{$\dagger$} & \textbf{0.52}\textsuperscript{$\ddagger$} & \textbf{100.00} & \textbf{100.00} & 
\textbf{0.00} & \textbf{95.75}\textsuperscript{$\ddagger$} & \textbf{90.94}\textsuperscript{$\ddagger$} & \textbf{13.85}\textsuperscript{$\ddagger$} & \textbf{100.00} & \textbf{100.00} & \textbf{0.00} & \textbf{91.09}\textsuperscript{$\ddagger$} & \textbf{94.48}\textsuperscript{$\ddagger$} & \textbf{21.11}\textsuperscript{$\ddagger$} \\
& {\scriptsize$\pm$0.51} & {\scriptsize$\pm$0.25} & {\scriptsize$\pm$1.77}
& {\scriptsize$\pm$0.14} & {\scriptsize$\pm$0.08} & {\scriptsize$\pm$0.60}
& {\scriptsize$\pm$0.00} & {\scriptsize$\pm$0.00} & {\scriptsize$\pm$0.00}
& {\scriptsize$\pm$0.76} & {\scriptsize$\pm$1.62} & {\scriptsize$\pm$0.55}
& {\scriptsize$\pm$0.00} & {\scriptsize$\pm$0.00} & {\scriptsize$\pm$0.00}
& {\scriptsize$\pm$0.25} & {\scriptsize$\pm$0.29} & {\scriptsize$\pm$0.37} \\

\bottomrule
\end{tabular}
\end{adjustbox}

\label{tab:results_rfmid}
\end{table*}

\begin{table*}[t]
\centering
\scriptsize
\setlength{\tabcolsep}{1pt}
\caption{Comparison of OOD detection results of our method with those of state-of-the-art methods on the ISIC2019 dataset (mean $\pm$ standard deviation). The compared methods are same as those listed in table~\ref{tab:results_rfmid}. Performance is measured using AUROC, AUPR, and FPR95, higher AUROC and AUPR, and lower FPR95 indicate superior performance. The best results are highlighted in \textbf{bold}, while the second-best results are \uline{underlined}. $[\ddagger]$ $p < 0.01$, $[\dagger]$ $p < 0.05$; paired $t$-test with respect to the top baseline result.}
\begin{adjustbox}{max width=\textwidth}
\begin{tabular}{l|ccc|ccc|ccc|ccc|ccc|ccc}
\toprule
 & \multicolumn{18}{c}{\textbf{ISIC 2019}} \\
\cmidrule(lr){2-19}
\multirow{2}{*}{\textbf{Method}} 
& \multicolumn{3}{c}{\textbf{Open Set}} 
& \multicolumn{3}{c}{\textbf{NearOOD1}} 
& \multicolumn{3}{c}{\textbf{NearOOD2}} 
& \multicolumn{3}{c}{\textbf{Corruptions}} 
& \multicolumn{3}{c}{\textbf{FarOOD}} 
& \multicolumn{3}{c}{\textbf{Average}} \\
\cmidrule(lr){2-4}\cmidrule(lr){5-7}\cmidrule(lr){8-10}
\cmidrule(lr){11-13}\cmidrule(lr){14-16}\cmidrule(lr){17-19}
& AUROC & AUPR & FPR95
& AUROC & AUPR & FPR95
& AUROC & AUPR & FPR95
& AUROC & AUPR & FPR95
& AUROC & AUPR & FPR95
& AUROC & AUPR & FPR95 \\
\midrule

OE
& 44.81 & 55.97 & 96.44 & 80.38 & 89.35 & 77.43 & 74.50 & 86.11 & 79.08 & 84.27 & 78.65 & 61.00 & 55.39 & 62.17 & 96.66 & 67.87 & 74.45 & 82.12 \\
& {\scriptsize$\pm$1.83} & {\scriptsize$\pm$0.76} & {\scriptsize$\pm$1.15}
& {\scriptsize$\pm$0.50} & {\scriptsize$\pm$0.42} & {\scriptsize$\pm$2.56}
& {\scriptsize$\pm$2.73} & {\scriptsize$\pm$1.60} & {\scriptsize$\pm$6.14}
& {\scriptsize$\pm$5.95} & {\scriptsize$\pm$7.04} & {\scriptsize$\pm$13.50}
& {\scriptsize$\pm$12.44} & {\scriptsize$\pm$11.37} & {\scriptsize$\pm$2.69}
& {\scriptsize$\pm$3.61} & {\scriptsize$\pm$3.37} & {\scriptsize$\pm$2.29} \\

PASCL
& 44.69 & 55.08 & 95.38 & \uline{94.01} & \uline{96.27} & \uline{19.68} & \uline{84.00} & \uline{91.42} & \uline{57.98} & 83.43 & 74.44 & 52.56 & 98.59 & 99.13 & 1.08 & 80.94 & 83.27 & 45.34 \\
& {\scriptsize$\pm$0.49} & {\scriptsize$\pm$0.22} & {\scriptsize$\pm$0.35}
& {\scriptsize$\pm$1.26} & {\scriptsize$\pm$0.80} & {\scriptsize$\pm$8.73}
& {\scriptsize$\pm$2.57} & {\scriptsize$\pm$1.57} & {\scriptsize$\pm$10.08}
& {\scriptsize$\pm$1.52} & {\scriptsize$\pm$1.08} & {\scriptsize$\pm$15.31}
& {\scriptsize$\pm$0.70} & {\scriptsize$\pm$0.45} & {\scriptsize$\pm$1.21}
& {\scriptsize$\pm$0.71} & {\scriptsize$\pm$0.51} & {\scriptsize$\pm$6.58} \\

EAT
& 47.81 & 59.40 & 95.72 & 87.77 & 93.09 & 47.41 & 75.69 & 87.10 & 77.89 & 90.70 & 85.95 & 39.99 & 94.63 & 96.96 & 58.43 & 79.32 & 84.50 & 63.89 \\
& {\scriptsize$\pm$0.32} & {\scriptsize$\pm$0.39} & {\scriptsize$\pm$0.17}
& {\scriptsize$\pm$5.37} & {\scriptsize$\pm$3.26} & {\scriptsize$\pm$11.78}
& {\scriptsize$\pm$3.61} & {\scriptsize$\pm$1.99} & {\scriptsize$\pm$5.36}
& {\scriptsize$\pm$1.07} & {\scriptsize$\pm$0.97} & {\scriptsize$\pm$9.25}
& {\scriptsize$\pm$1.45} & {\scriptsize$\pm$0.78} & {\scriptsize$\pm$35.31}
& {\scriptsize$\pm$1.84} & {\scriptsize$\pm$1.05} & {\scriptsize$\pm$7.99} \\

PATT 
& \uline{53.35} & \uline{62.35} & \uline{95.78} & 62.18 & 80.32 & 97.24 & 74.55 & 84.95 & 81.56 & 53.98 & 52.06 & 99.66 & 85.37 & 91.42 & 93.40 & 66.09 & 74.82 & 93.53 \\
& {\scriptsize$\pm$4.75} & {\scriptsize$\pm$4.79} & {\scriptsize$\pm$1.28}
& {\scriptsize$\pm$6.33} & {\scriptsize$\pm$3.24} & {\scriptsize$\pm$1.55}
& {\scriptsize$\pm$0.91} & {\scriptsize$\pm$0.90} & {\scriptsize$\pm$2.16}
& {\scriptsize$\pm$11.33} & {\scriptsize$\pm$10.20} & {\scriptsize$\pm$0.09}
& {\scriptsize$\pm$10.44} & {\scriptsize$\pm$6.28} & {\scriptsize$\pm$11.33}
& {\scriptsize$\pm$5.21} & {\scriptsize$\pm$3.88} & {\scriptsize$\pm$1.75} \\

COCL
& 49.18 & 58.67 & 94.83 & 91.21 & 95.22 & 40.28 & 79.68 & 87.59 & 59.27 & \uline{91.50} & \uline{86.72} & \uline{34.98} & \textbf{100.00} & \textbf{100.00} & \textbf{0.00} & \uline{82.31} & \uline{85.64} & \uline{45.87} \\
& {\scriptsize$\pm$1.19} & {\scriptsize$\pm$1.41} & {\scriptsize$\pm$0.05}
& {\scriptsize$\pm$1.94} & {\scriptsize$\pm$1.13} & {\scriptsize$\pm$6.86}
& {\scriptsize$\pm$1.29} & {\scriptsize$\pm$1.80} & {\scriptsize$\pm$3.35}
& {\scriptsize$\pm$1.76} & {\scriptsize$\pm$1.76} & {\scriptsize$\pm$13.46}
& {\scriptsize$\pm$0.00} & {\scriptsize$\pm$0.00} & {\scriptsize$\pm$0.00}
& {\scriptsize$\pm$0.45} & {\scriptsize$\pm$0.50} & {\scriptsize$\pm$3.11} \\

MARVEL
& \textbf{54.04} & \textbf{63.49} & \textbf{94.68} & \textbf{98.78}\textsuperscript{$\ddagger$} & \textbf{99.06}\textsuperscript{$\ddagger$} & \textbf{3.45}\textsuperscript{$\ddagger$} & \textbf{86.95}\textsuperscript{$\dagger$} & \textbf{92.82}\textsuperscript{$\ddagger$} & \textbf{49.78}\textsuperscript{$\ddagger$} & \textbf{96.39}\textsuperscript{$\ddagger$} & \textbf{93.57}\textsuperscript{$\ddagger$} & \textbf{16.33}\textsuperscript{$\ddagger$} & \textbf{100.00} & \textbf{100.00} & \textbf{0.00} & \textbf{87.23}\textsuperscript{$\ddagger$} & \textbf{89.79}\textsuperscript{$\ddagger$} & \textbf{32.85}\textsuperscript{$\ddagger$} \\
& {\scriptsize$\pm$0.82} & {\scriptsize$\pm$1.10} & {\scriptsize$\pm$0.20}
& {\scriptsize$\pm$0.67} & {\scriptsize$\pm$0.55} & {\scriptsize$\pm$1.63}
& {\scriptsize$\pm$0.17} & {\scriptsize$\pm$0.39} & {\scriptsize$\pm$1.21}
& {\scriptsize$\pm$0.18} & {\scriptsize$\pm$0.34} & {\scriptsize$\pm$2.53}
& {\scriptsize$\pm$0.00} & {\scriptsize$\pm$0.00} & {\scriptsize$\pm$0.00}
& {\scriptsize$\pm$0.19} & {\scriptsize$\pm$0.21} & {\scriptsize$\pm$0.98} \\

\bottomrule
\end{tabular}
\end{adjustbox}
\label{tab:results_isic}
\end{table*}

From Table~\ref{tab:results_isic}, \texttt{MARVEL} achieves the best overall performance compared to the second strongest baseline COCL (with 4.92\% less AUROC).
However, unlike \texttt{MARVEL} which maintains strong performance across the entire OOD spectrum, other methods tend to perform well only in specific settings. 
For example, PASCL achieves the second-best performance in the NearOOD settings (with AUROC 4.77\%$\downarrow$ in NearOOD1 and 2.95\%$\downarrow$ in NearOOD2 than \texttt{MARVEL}), while COCL performs strongly on Corruptions (with AUROC 4.89\%$\downarrow$ than \texttt{MARVEL}). 
For open-set detection, \texttt{MARVEL} still achieves the best performance, while the second-best result is obtained by PATT.

\begin{table*}[t]
\centering
\scriptsize
\setlength{\tabcolsep}{1pt}
\caption{Comparison of OOD detection results of our method with those of state-of-the-art methods on the NCTCRC dataset (mean $\pm$ standard deviation). The compared methods are same as those listed in table~\ref{tab:results_rfmid}. Performance is measured using AUROC, AUPR, and FPR95, higher AUROC and AUPR, and lower FPR95 indicate superior performance. The best results are highlighted in \textbf{bold}, while the second-best results are \uline{underlined}. $[\ddagger]$ $p < 0.01$, $[\dagger]$ $p < 0.05$; paired $t$-test with respect to the top baseline result.}

\begin{adjustbox}{max width=\textwidth}
\begin{tabular}{l|ccc|ccc|ccc|ccc|ccc|ccc}
\toprule
\multirow{2}{*}{\textbf{Method}}
& \multicolumn{18}{c}{\textbf{NCTCRC}} \\
\cmidrule(lr){2-19}
& \multicolumn{3}{c}{\textbf{Open Set}}
& \multicolumn{3}{c}{\textbf{NearOOD1}}
& \multicolumn{3}{c}{\textbf{NearOOD2}}
& \multicolumn{3}{c}{\textbf{Corruptions}}
& \multicolumn{3}{c}{\textbf{FarOOD}}
& \multicolumn{3}{c}{\textbf{Average}} \\
\cmidrule(lr){2-4}\cmidrule(lr){5-7}\cmidrule(lr){8-10}
\cmidrule(lr){11-13}\cmidrule(lr){14-16}\cmidrule(lr){17-19}
& AUROC & AUPR & FPR95
& AUROC & AUPR & FPR95
& AUROC & AUPR & FPR95
& AUROC & AUPR & FPR95
& AUROC & AUPR & FPR95
& AUROC & AUPR & FPR95 \\
\midrule

OE 
& 61.02 & 83.77 & 91.55 & 80.51 & 98.61 & 79.72 & 83.54 & 77.46 & 55.84 & 69.48 & 70.19 & 80.23 & 67.64 & 92.78 & 86.84 & 72.44 & 84.56 & 78.84 \\
& {\scriptsize$\pm$3.74} & {\scriptsize$\pm$2.49} & {\scriptsize$\pm$1.10}
& {\scriptsize$\pm$6.34} & {\scriptsize$\pm$0.63} & {\scriptsize$\pm$15.35}
& {\scriptsize$\pm$10.87} & {\scriptsize$\pm$14.20} & {\scriptsize$\pm$32.42}
& {\scriptsize$\pm$4.07} & {\scriptsize$\pm$3.45} & {\scriptsize$\pm$6.13}
& {\scriptsize$\pm$10.55} & {\scriptsize$\pm$2.92} & {\scriptsize$\pm$9.48}
& {\scriptsize$\pm$0.58} & {\scriptsize$\pm$1.42} & {\scriptsize$\pm$6.45} \\

PASCL 
& 53.58 & 80.06 & \uline{91.24} & 93.01 & 99.56 & \uline{39.39} & 83.64 & 78.26 & 56.84 & 79.44 & 81.28 & 71.88 & 98.34 & 99.75 & 10.65 & 81.60 & 87.78 & 54.00 \\
& {\scriptsize$\pm$2.84} & {\scriptsize$\pm$1.31} & {\scriptsize$\pm$5.02}
& {\scriptsize$\pm$5.36} & {\scriptsize$\pm$0.36} & {\scriptsize$\pm$27.80}
& {\scriptsize$\pm$6.56} & {\scriptsize$\pm$7.03} & {\scriptsize$\pm$19.95}
& {\scriptsize$\pm$5.26} & {\scriptsize$\pm$3.16} & {\scriptsize$\pm$20.16}
& {\scriptsize$\pm$2.37} & {\scriptsize$\pm$0.36} & {\scriptsize$\pm$16.73}
& {\scriptsize$\pm$1.81} & {\scriptsize$\pm$1.26} & {\scriptsize$\pm$12.45} \\

EAT 
& 62.49 & 84.49 & 91.59 & \uline{94.90} & \uline{99.71} & 44.46 & \textbf{93.08} & \textbf{93.07} & \uline{56.47} & 75.39 & 71.60 & 77.91 & 89.64 & 98.42 & 96.51 & \uline{83.10} & \uline{89.46} & \uline{73.39} \\
& {\scriptsize$\pm$1.87} & {\scriptsize$\pm$0.45} & {\scriptsize$\pm$1.96}
& {\scriptsize$\pm$1.48} & {\scriptsize$\pm$0.08} & {\scriptsize$\pm$24.60}
& {\scriptsize$\pm$0.30} & {\scriptsize$\pm$0.60} & {\scriptsize$\pm$12.49}
& {\scriptsize$\pm$2.81} & {\scriptsize$\pm$5.05} & {\scriptsize$\pm$6.73}
& {\scriptsize$\pm$1.34} & {\scriptsize$\pm$0.20} & {\scriptsize$\pm$4.87}
& {\scriptsize$\pm$1.25} & {\scriptsize$\pm$1.00} & {\scriptsize$\pm$7.63} \\

PATT 
& \uline{55.81} & \uline{82.78} & 98.89 & 81.98 & 98.91 & 97.29 & 74.32 & 79.98 & 99.06 & 62.82 & 69.88 & 99.21 & 81.02 & 96.91 & 96.54 & 71.19 & 85.69 & 98.20 \\
& {\scriptsize$\pm$1.04} & {\scriptsize$\pm$1.47} & {\scriptsize$\pm$1.52}
& {\scriptsize$\pm$6.88} & {\scriptsize$\pm$0.44} & {\scriptsize$\pm$4.70}
& {\scriptsize$\pm$5.03} & {\scriptsize$\pm$2.77} & {\scriptsize$\pm$0.84}
& {\scriptsize$\pm$3.02} & {\scriptsize$\pm$1.86} & {\scriptsize$\pm$1.07}
& {\scriptsize$\pm$9.83} & {\scriptsize$\pm$1.67} & {\scriptsize$\pm$5.99}
& {\scriptsize$\pm$4.61} & {\scriptsize$\pm$1.00} & {\scriptsize$\pm$2.77} \\

COCL
& 52.39 & 79.67 & 92.21 & 91.77 & 99.46 & 42.93 & 84.95 & 82.06 & 64.96 & \uline{81.87} & \uline{80.66} & \uline{60.03} & \textbf{99.98} & \textbf{100.00} & \uline{0.09} & 82.19 & 88.37 & 52.05 \\
& {\scriptsize$\pm$4.79} & {\scriptsize$\pm$1.98} & {\scriptsize$\pm$5.22}
& {\scriptsize$\pm$4.01} & {\scriptsize$\pm$0.26} & {\scriptsize$\pm$22.07}
& {\scriptsize$\pm$6.98} & {\scriptsize$\pm$9.93} & {\scriptsize$\pm$12.21}
& {\scriptsize$\pm$1.26} & {\scriptsize$\pm$1.39} & {\scriptsize$\pm$12.50}
& {\scriptsize$\pm$0.03} & {\scriptsize$\pm$0.00} & {\scriptsize$\pm$0.13}
& {\scriptsize$\pm$2.60} & {\scriptsize$\pm$2.21} & {\scriptsize$\pm$7.69} \\

MARVEL
& \textbf{69.18}\textsuperscript{$\ddagger$} & \textbf{87.20}\textsuperscript{$\ddagger$} & \textbf{83.93}\textsuperscript{$\ddagger$} & \textbf{97.07}\textsuperscript{$\ddagger$} & \textbf{99.84} & \textbf{18.90}\textsuperscript{$\ddagger$} & \uline{92.58} & \uline{92.40} & \textbf{56.09} & \textbf{94.77}\textsuperscript{$\ddagger$} & \textbf{93.96}\textsuperscript{$\ddagger$} & \textbf{22.55}\textsuperscript{$\ddagger$} & \uline{99.86} & \uline{99.98} & \textbf{0.00} & \textbf{90.49}\textsuperscript{$\ddagger$} & \textbf{94.67}\textsuperscript{$\ddagger$} & \textbf{36.49}\textsuperscript{$\ddagger$} \\
& {\scriptsize$\pm$3.67} & {\scriptsize$\pm$2.25} & {\scriptsize$\pm$0.18}
& {\scriptsize$\pm$0.93} & {\scriptsize$\pm$0.05} & {\scriptsize$\pm$14.27}
& {\scriptsize$\pm$2.61} & {\scriptsize$\pm$3.67} & {\scriptsize$\pm$5.47}
& {\scriptsize$\pm$0.87} & {\scriptsize$\pm$1.83} & {\scriptsize$\pm$2.45}
& {\scriptsize$\pm$0.11} & {\scriptsize$\pm$0.02} & {\scriptsize$\pm$0.00}
& {\scriptsize$\pm$1.28} & {\scriptsize$\pm$1.53} & {\scriptsize$\pm$3.25} \\

\bottomrule
\end{tabular}
\end{adjustbox}

\label{tab:results_nctcrc}
\end{table*}

\texttt{MARVEL} achieves the best overall performance on the NCTCRC dataset, as reported in Table~\ref{tab:results_nctcrc}, with EAT being the  second-best baseline (with average AUROC 7\% lower than \texttt{MARVEL}). \texttt{MARVEL} shows significant improvements in the Open-Set, NearOOD1, and Corruption settings (with p-values $<$ 0.01 for all metrics except AUPR in NearOOD1). On NearOOD2 and FarOOD, \texttt{MARVEL} provides on-par performance compared to baselines (with all performances close to saturation value of 100.00, making the difference practically negligible). However, \texttt{MARVEL} still achieves a lower FPR95 in this setting, indicating better control of false positives despite the small AUROC gap.
For open-set detection, \texttt{MARVEL} achieves 13\% higher AUROC over PATT.

\subsection{Comparison of the proposed ID classification with state-of-the-art methods}

\begin{table*}[t]
\centering
\caption{Comparison of the ID classification performance of our method with those of state-of-the-art methods on all three datasets (mean $\pm$ standard deviation). The compared methods are same as those listed in table~\ref{tab:results_rfmid}. Results are reported in terms of Accuracy and Balanced Accuracy, with higher values indicating better performance. The best results are highlighted in \textbf{bold}, and the second-best results are \underline{underlined}. $[\ddagger]$ $p < 0.01$, $[\dagger]$ $p < 0.05$; paired $t$-test with respect to the top baseline result.}
\label{tab:overall_acc_bacc}
\begin{adjustbox}{max width=\textwidth}
\begin{tabular}{l|cc cc cc}
\hline
\multirow{2}{*}{Method} 
& \multicolumn{2}{c}{RFMiD}
& \multicolumn{2}{c}{ISIC2019}
& \multicolumn{2}{c}{NCT-CRC} \\
& Acc & BAcc & Acc & BAcc & Acc & BAcc \\
\hline
OE
& 36.30 {$\pm$ 3.03} & 19.65{$\pm$ 1.02}
& 61.62{$\pm$ 3.41} & 43.69{$\pm$ 4.75}
& 32.52{$\pm$ 10.97} & 40.38{$\pm$ 11.08} \\

PASCL
& 48.52{$\pm$ 1.60} & 21.28{$\pm$ 2.04}
& \underline{70.67{$\pm$ 2.32}} & 50.68{$\pm$ 3.06}
& 68.30{$\pm$ 1.45} & 64.43{$\pm$ 1.26} \\

EAT
& 32.11{$\pm$ 3.27} & 13.68{$\pm$ 0.85}
& 46.12{$\pm$ 5.13} & 44.92{$\pm$ 2.29}
& 51.99{$\pm$ 11.03} & 47.67{$\pm$ 14.07} \\

PATT 
& \underline{60.30{$\pm$ 1.24}} & 41.54{$\pm$ 3.54}
& 60.04{$\pm$ 4.55} & 62.36{$\pm$ 4.21}
& \underline{71.75{$\pm$ 1.46}} & \underline{84.36{$\pm$ 3.51}} \\

COCL
& 58.03{$\pm$ 5.52} & \underline{43.94{$\pm$ 2.78}}
& 63.09{$\pm$ 3.96} & \underline{66.19{$\pm$ 1.17}}
& 68.27{$\pm$ 1.53} & 81.66{$\pm$ 2.14} \\

MARVEL
& \textbf{66.49{$\pm$ 0.45}}\textsuperscript{$\ddagger$} & \textbf{50.52{$\pm$ 2.00}}\textsuperscript{$\ddagger$}
& \textbf{72.88{$\pm$ 0.85}}\textsuperscript{$\ddagger$} & \textbf{67.18{$\pm$ 0.77}}\textsuperscript{$\dagger$}
& \textbf{77.02{$\pm$ 0.46}}\textsuperscript{$\ddagger$} & \textbf{89.38{$\pm$ 0.41}}\textsuperscript{$\ddagger$} \\
\hline
\end{tabular}
\end{adjustbox}
\end{table*}

Table~\ref{tab:overall_acc_bacc} summarises ID classification performance across all three datasets. \texttt{\texttt{MARVEL}} consistently achieves the highest accuracy and balanced accuracy (significantly higher than baselines with p-value $<$ 0.01), indicating strong overall performance. On RFMiD, while baseline methods show heavy disparity between Accuracy and balanced accuracy, reflecting strong performance on frequent head classes but poor recognition of rare tail classes. In contrast, \texttt{\texttt{MARVEL}} simultaneously improves both accuracies, reducing the disparity of performance in head and tail classes. Specifically, for RFMiD, \texttt{MARVEL} improves accuracy by 6.19\% and balanced accuracy by 6.58\% compared to second-best baselines. Similarly in ISIC2019, \texttt{MARVEL} provides 2.21\% and 0.99\% improvements, and in NCTCRC, 5.27\% and 5.02\% improvements in Accuracy and balanced accuracy respectively.
Overall, these results demonstrate that high accuracy alone can be misleading in long-tailed settings, as it may arise from overfitting to head classes, whereas \texttt{\texttt{MARVEL}} consistently improves both accuracy and balanced accuracy.

\begin{center}
\nopagebreak
\includegraphics[width=\linewidth]{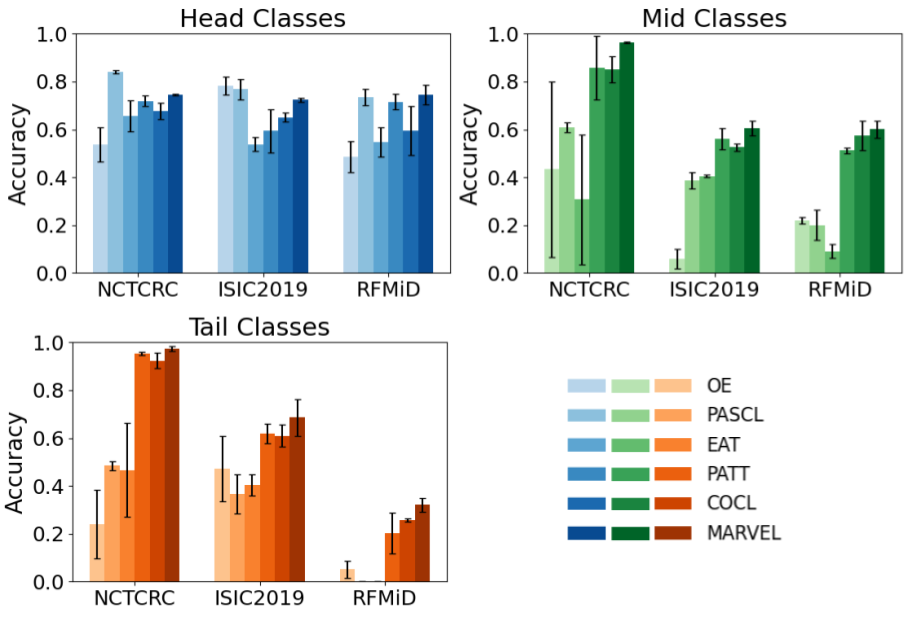}
\captionof{figure}{Head, mid, tail class accuracy comparison across NCTCRC, ISIC2019, and RFMiD datasets. Bars show mean accuracy and standard deviation over three runs.}
\label{fig:head-tail-accu}
\end{center}

Figure~\ref{fig:head-tail-accu} breaks down ID classification performance on RFMiD, ISIC2019 and NCTCRC datasets across head, mid, and tail classes, providing a fine-grained view of model behaviour under long-tailed distributions. Across all datasets, most methods achieve relatively strong head-class accuracy, with modest differences between approaches, indicating that frequent classes are generally easy to classify. In contrast, performance gaps widen substantially for mid and tail classes, where imbalance effects are most pronounced. In particular, several baselines that perform well on head classes exhibit a sharp degradation on tail classes, for example, OE and PASCL show almost 0 tail accuracies on RFMiD, highlighting overfitting to dominant classes. \texttt{MARVEL} consistently mitigates this effect, achieving the highest mid- and tail-class accuracies across all datasets, on NCTCRC the tail accuracy approaches 95\%. These improvements align with the balanced accuracy trends observed in Table~\ref{tab:overall_acc_bacc}, confirming that \texttt{\texttt{MARVEL}}’s gains stem not from improved head-class fitting alone, but from substantially enhanced recognition of under-represented classes, which is critical for robust performance in long-tailed medical classification.

\subsection{Ablation study I: Impact of the choice of the classifier}

Table~\ref{tab:ablation-classifier-experts} analyses the impact of classifier design within \texttt{MARVEL} on both its ID classification and OOD detection performances. For ID classifiers (A), we observe a consistent performance hierarchy across all datasets: while FC and cosine classifiers provide modest gains (e.g., RFMiD ACC improves marginally from 48.16 to 50.26), the adoption of a vMF-based classifier leads to a substantial jump in both ID accuracy and OOD detection. 
Incorporating the proposed NvMF classifier yields the strongest improvements, boosting accuracy (RFMiD: 7.33\%$\uparrow$, ISIC2019: 2.59\%$\uparrow$ and NCTCRC: 1.03\%$\uparrow$) and AUROC (RFMiD: 8.40\%$\uparrow$, ISIC2019: 4.3\%$\uparrow$ and NCTCRC: 0.24\%$\uparrow$) relative to vMF. 
This demonstrates the benefit of modelling class-conditional angular distributions with greater flexibility. For OOD classifiers (B), removing the outlier expert results in a marked degradation in OOD detection (e.g., AUROC drops to 79.63 on RFMiD), showing its contribution. While FC-based OOD experts achieve strong performance, the NvMF classifier remains competitive and consistently balances ID accuracy and OOD detection across datasets, indicating that angularly structured decision boundaries are well suited for both ID discrimination and outlier separation within \texttt{MARVEL}.

\begin{center}
\setlength{\tabcolsep}{2pt}
\captionof{table}{Ablation study on the effect of classifier design for ID and OOD prediction across datasets. Results are reported in terms of classification accuracy (ACC) and OOD detection performance (AUROC) on RFMiD, ISIC2019, and NCTCRC, compared with FC \citep{alexnet}, Cosine \citep{Gidaris_2018_CVPR}, vMF \citep{ming2023cider}}
\label{tab:ablation-classifier-experts}

\begin{adjustbox}{max width=0.45\textwidth}
\begin{tabular}{l cc cc cc}
\toprule
 & \multicolumn{2}{c}{RFMiD} & \multicolumn{2}{c}{ISIC2019} & \multicolumn{2}{c}{NCTCRC} \\
\cmidrule(lr){2-3}\cmidrule(lr){4-5}\cmidrule(lr){6-7}
Setting & ACC & AUROC & ACC & AUROC & ACC & AUROC \\
\midrule
\multicolumn{7}{l}{\textbf{A. ID Classifiers}} \\
\quad FC      & 48.16 & 78.30 & 62.56 & 73.35 & 73.71 & 85.61 \\
\quad Cosine  & 50.26 & 79.05 & 62.40 & 72.73 & 74.51 & 86.89 \\
\quad vMF     & 58.63 & 82.96 & 69.66 & 81.64 & 76.46 & 88.89 \\
\quad NvMF    & 65.96 & 91.36 & 72.25 & 85.94 & 77.49 & 89.13 \\
\midrule
\multicolumn{7}{l}{\textbf{B. OOD Classifiers}} \\
\quad None    & 58.63 & 79.63 & 65.09 & 74.81 & 67.77 & 80.28 \\
\quad Cosine  & 64.13 & 84.24 & 71.17 & 82.43 & 75.26 & 90.48 \\
\quad vMF     & 62.04 & 85.35 & 71.36 & 81.64 & 74.96 & 89.16 \\
\quad NvMF    & 65.44 & 88.53 & 71.40 & 83.18 & 74.51 & 90.66 \\
\quad FC      & 66.75 & 88.70 & 72.54 & 85.71 & 76.30 & 91.11 \\

\bottomrule
\end{tabular}
\end{adjustbox}
\end{center}

\subsection{Ablation study II: Effect of the number of experts within the ensemble}

Figure~\ref{fig:experts_ensemble} shows the impact of number of experts within the ensemble on ID classification and OOD detection across the three datasets.
Increasing the number of experts yields consistent performance gains up to 3 experts.
For ID classification (left), moving from a single expert to three experts improves accuracy on RFMiD ($\approx$65\% $\to$ 67.5\%), ISIC2019 ($\approx$67.5\% $\to$ 73\%), and NCTCRC ($\approx$73\% $\to$ 77\%).
Similar trends are observed for OOD detection (right), where AUROC improves on RFMiD ($\approx$86 $\to$ 91), ISIC2019 ($\approx$80 $\to$ 88) and NCTCRC ($\approx$89 $\to$ 92) when increasing to three experts.
Adding a fourth expert does not provide further improvement and, in most cases, leads to a slight degradation in OOD performance, with AUROC dropping by $\approx2\%$ on RFMiD, $\approx5\%$ on ISIC2019, and $\approx4\%$ on NCTCRC.
Based on these trends, an ensemble of three experts offers the best trade-off between performance and complexity and is therefore used as the default setting in all subsequent experiments.

\begin{center}
    \includegraphics[width=\linewidth]{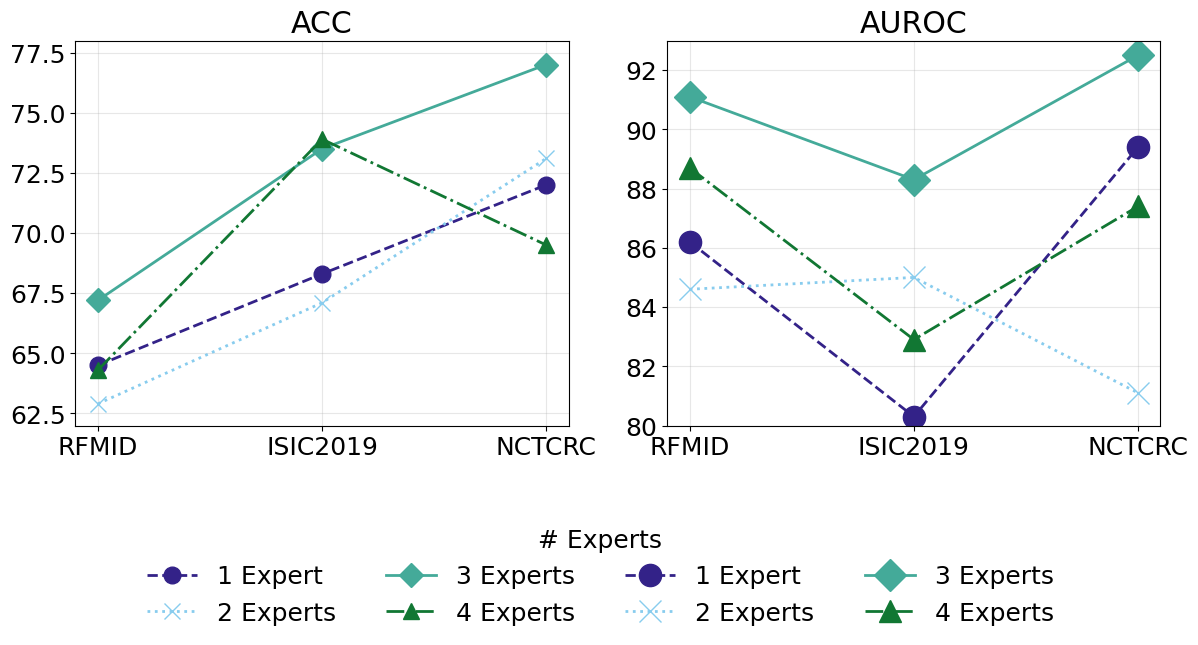}
    \captionof{figure}{Effect of expert ensembling on performance across datasets. Accuracy (ACC, left) and AUROC (right) are reported for 1 to 4 experts. Performance improves with more experts, with the ensemble of three experts achieving the best overall results across RFMID, ISIC2019, and NCTCRC datasets.}
    \label{fig:experts_ensemble}
\end{center}

\subsection{Confidence calibration analysis}

\begin{center}
    \includegraphics[width=\linewidth]{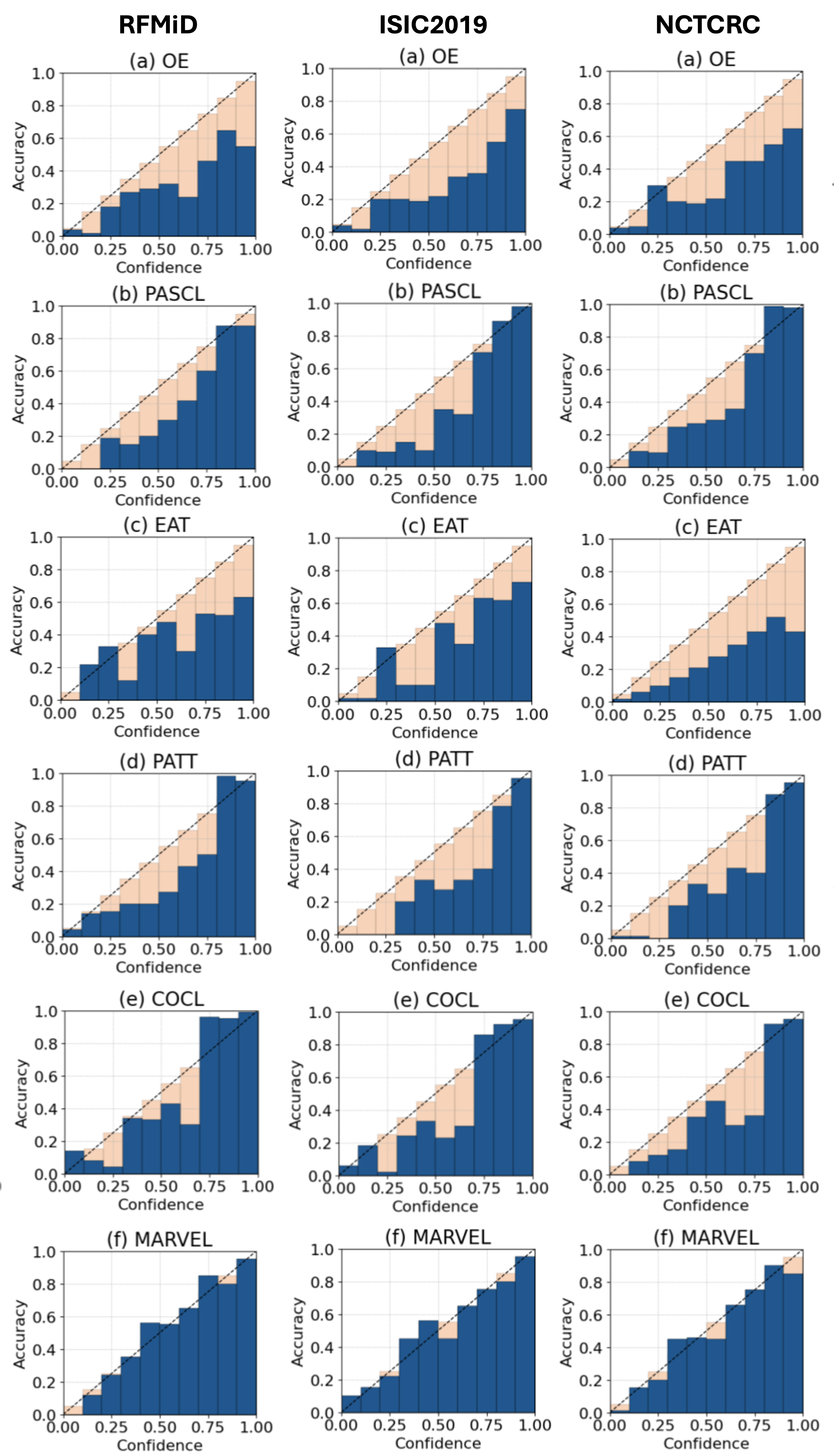}
    \captionof{figure}{Reliability diagrams comparing confidence calibration across different methods on (left to right) RFMID, ISIC2019, and NCTCRC datasets: (top to bottom) (a) OE, (b) PASCL, (c) EAT, (d) PATT, (e) COCL, and (f) \texttt{MARVEL} (proposed). The blue bars indicate the empirical accuracy within each confidence bin, while the dashed diagonal denotes perfect calibration. The shaded red region highlights calibration error.}
    \label{fig:calibration}
\end{center}

Figure~\ref{fig:calibration} presents reliability diagrams comparing the calibration behaviour of the evaluated methods including the proposed \texttt{MARVEL} on the evaluation datasets.
Each plot depicts the empirical accuracy within the confidence bins against the predicted confidence values, where the diagonal line corresponds to perfect calibration.
Deviations from this diagonal indicate miscalibration, with the shaded area highlighting the magnitude of calibration error.
Across the baseline methods, we observe varying degrees of miscalibration.
In particular, PASCL, PATT, and COCL tend to produce overconfident predictions, where the predicted confidence exceeds the empirical accuracy, while OE and EAT exhibit under-confident behaviour, where the predicted confidence underestimates the true accuracy.
In contrast, \texttt{MARVEL} demonstrates more consistent alignment between confidence and empirical accuracy, indicating improved calibration.
These observations remain consistent across the NCTCRC, ISIC2019, and RFMiD datasets. 
Specifically, the overconfidence gap in methods such as PASCL and PATT persists regardless of the underlying data distribution.
Furthermore, while EAT exhibits more pronounced under-confidence on the NCTCRC dataset compared to the other two benchmarks, MARVEL consistently minimises these deviations and maintains a tight alignment with the diagonal. 
This improvement can be attributed to the proposed expert-based architecture, where the ensemble of NvMF classifiers produces more reliable confidence estimates by aggregating diverse predictions. Such aggregation mitigates overconfident errors from individual experts and leads to smoother, better-calibrated decision boundaries, while maintaining strong ID classification and OOD detection performance.

\subsection{Risk-Coverage analysis}
\begin{center}
    \includegraphics[width=0.45\textwidth]{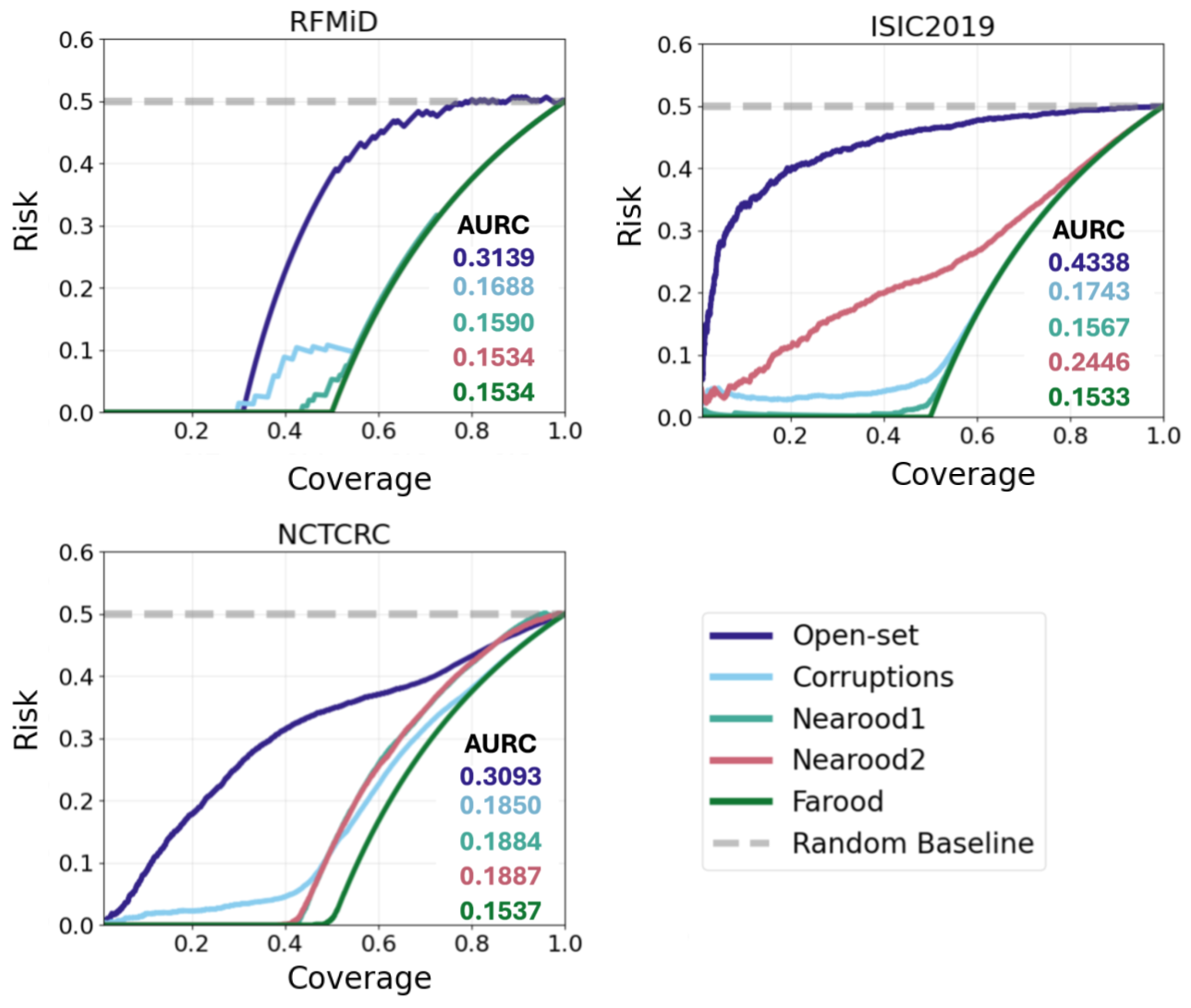}
    \captionof{figure}{Risk-coverage curves for OOD detection across RFMiD, ISIC2019, and NCTCRC datasets. The x-axis represents coverage and the y-axis denotes risk and the individual legends show the AURC (area under the risk-coverage curve) values in the corresponding plot colours.}
    \label{fig:risk-coverage}
\end{center}

Across datasets, RFMiD demonstrates the most favorable behaviour. FarOOD, NearOOD1, and NearOOD2 curves largely overlap and maintain near-zero risk up to approximately 0.5 coverage, indicating strong separation between ID and these OOD samples. 
Open-set and corruption scenarios are comparatively more challenging; however, they still achieve near-zero risk up to around 0.3 coverage, followed by a steady increase. On ISIC2019, performance degrades notably for Open-set and NearOOD2 scenarios, which exhibit higher risk across coverage levels. 
Corruptions also show increased difficulty, though risk remains low below 0.5 coverage. 
In contrast, NearOOD1 performs well, and FarOOD achieves near-perfect behaviour with minimal risk up to 0.5 coverage. For NCTCRC, Open-set remains the most challenging setting, exhibiting consistently higher risk. Corruptions and NearOOD2 maintain low risk up to approximately 0.4 coverage, after which risk increases more sharply, with NearOOD2 showing a slightly steeper rise. FarOOD and NearOOD1 again demonstrate near-ideal behaviour, maintaining very low risk over a substantial coverage range. 

\subsection{Effect of the choice of the OOD detector}

\begin{center}
\captionof{table}{Ablation study on different OOD detector types across the three datasets. The compared methods include KNN \citep{sun2022knnood}, MD \citep{NEURIPS2018_abdeb6f5}, NN-guide \citep{Park_2023_ICCV}, MLS \citep{vaze2022openset}, EBO \citep{liu2020energy}, MSP \citep{hendrycks17baseline}}
\label{tab:ablation-ood-detectors}
\resizebox{\linewidth}{!}{
\begin{tabular}{l cc cc cc}
\toprule
 & \multicolumn{2}{c}{RFMiD} & \multicolumn{2}{c}{ISIC2019} & \multicolumn{2}{c}{NCTCRC} \\
\cmidrule(lr){2-3}\cmidrule(lr){4-5}\cmidrule(lr){6-7}
Detector & AUROC & FPR95 & AUROC & FPR95 & AUROC & FPR95 \\
\midrule
KNN     & 69.82 & 85.11 & 69.90 & 91.15 & 72.27 & 72.05 \\
MD     & 71.50 & 81.30 & 68.20 & 94.35 & 73.08 & 79.56 \\
NN-guide & 76.55 & 57.33 & 77.29 & 71.13 & 80.08 & 56.58 \\
MLS  & 86.89 & 26.71 & 81.64 & 34.69 & 84.62 & 78.70 \\
EBO  & 88.13 & 24.66 & 83.01 & 34.05 & 85.10 & 73.93 \\
MSP (Ours) & 90.87 & 21.47 & 85.16 & 31.34 & 91.68 & 41.44 \\
\bottomrule
\end{tabular}}
\end{center}

Table~\ref{tab:ablation-ood-detectors} evaluates the effect of different OOD scoring mechanisms within the \texttt{MARVEL} framework. Overall, logit-based detectors consistently outperform distance-based alternatives across all datasets, both in terms of AUROC and FPR95.
Among the detectors, MSP consistently achieves the highest AUROC across all datasets (RFMiD: 90.87, ISIC2019: 85.16, and NCTCRC: 91.68).
EBO achieves the second-best performance, with AUROC values consistently lower by approximately $2.74\%$ (RFMiD), $2.15\%$ (ISIC2019), and $6.58\%$ (NCTCRC) compared to MSP.
In contrast, distance-based approaches such as KNN and MD perform poorly, with AUROC dropping by approximately $20\%$ on RFMiD, $15\%$ on ISIC2019, and $18\%$ on NCTCRC compared to MSP.
NN-guided detectors provide moderate improvements over purely distance-based methods but still lags behind classifier-based detectors (MSP, MLS, EBO).

\subsection{Qualitative analysis of learned representations}

\begin{figure*}
    \centering
    \includegraphics[width=\linewidth]{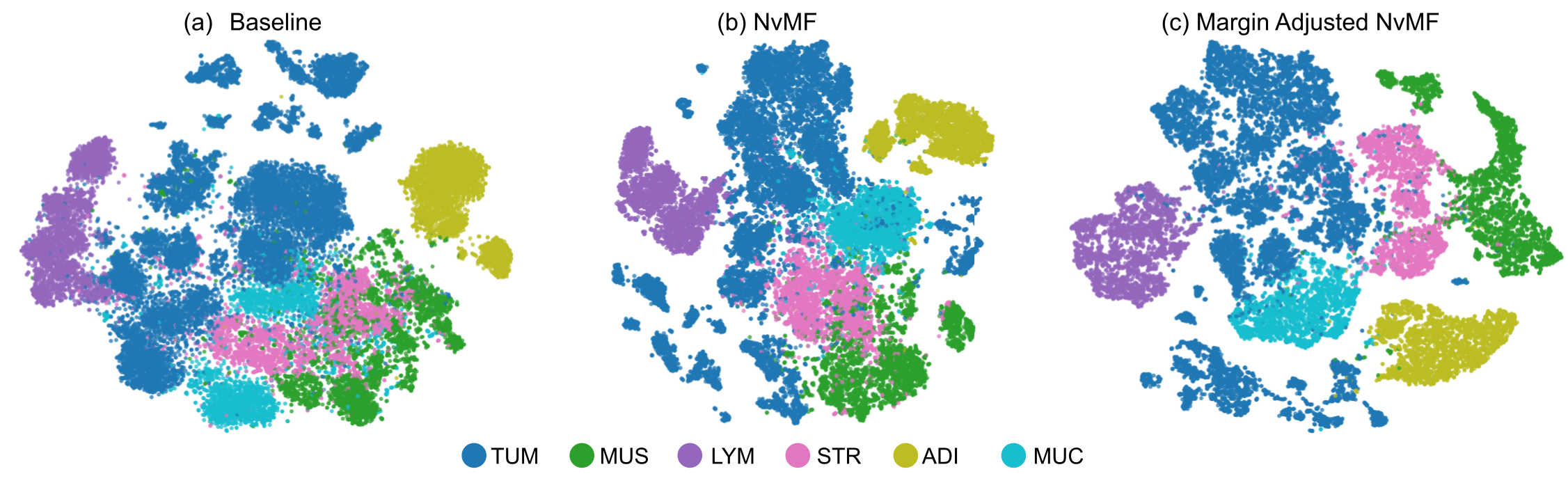}
    \caption{t-SNE visualization of learned feature representations in the NCTCRC dataset. Representation obtained using (a) vMF classifier (b) only NvMF classifier (c) the proposed margin-aware NvMF expert ensemble.}
    \label{fig:tsne}
\end{figure*}

Figure \ref{fig:tsne} illustrates the progressive improvement in feature separability across the three settings. 
Compared to the baseline (a), the NvMF classifier (b) produces more compact and better-defined clusters, as evident in classes such as LYM, with reduced overlap also observed for MUC and STR.
The most significant change, however, is observed in (c), where the margin-aware NvMF experts enforce visibly larger inter-cluster separation.
Classes such as MUC, STR, and MUS exhibit clearer boundaries and increased spacing from neighbouring clusters, indicating that margin-awareness actively pushes class representations apart.
This suggests that the proposed framework not only improves clustering quality but also explicitly enhances class discrimination by enforcing stronger geometric separation in the feature space.

\section{Discussion}
In this study, we present \texttt{MARVEL}, an integrated framework for long-tailed OOD detection that integrates a novel NvMF classifier, a label-distribution-aware multi-expert design for ID classification, and an explicit outlier expert for disentangling ID from OOD signals.
Across the spectrum of OOD settings, \texttt{MARVEL} consistently achieves strong performance, while the existing methods exhibit significant variability and lower performance across datasets and OOD regimes.
For instance, EAT performs competitively on Open-Set but degrades to the worst performance on FarOOD for the RFMiD dataset (Table \ref{tab:results_rfmid}), PATT ranks second on Open-Set yet performs the worst under corruption-based shifts on ISIC2019 (Table \ref{tab:results_isic}), and COCL achieves the best results on FarOOD but the worst on Open-Set OOD (Table \ref{tab:results_nctcrc}).
This inconsistency suggests that many prior methods are implicitly tailored to specific assumptions about OOD data.
For example, COCL formulates OOD detection as a classification problem using a single outlier class, which can lead to overfitting to the auxiliary OOD distribution seen during training and limits generalisation to unseen OOD types.
Similarly, EAT relies on augmentation-driven diversity and multiple synthetic outlier classes, which are effective when OOD samples resemble these generated variations but fail under structurally different farOOD scenarios.
As a result, these methods lack robustness across heterogeneous distribution shifts.
In contrast, \texttt{MARVEL} maintains stable performance by combining specialised experts for different regions of the label distribution with an explicit outlier expert, enabling it to handle both feature imbalance and distributional uncertainty in an integrated manner.
Notably, Open-Set detection remains the most challenging scenario, as evidenced by comparatively lower AUROC values across all methods.
This is expected, as Open-Set samples often lie close to the ID manifold and share semantic or visual similarities with tail classes, making them difficult to distinguish.
Even though \texttt{MARVEL} achieves the best performance in this setting, this is the only scenario where AUROC falls below $85\%$, highlighting the inherent difficulty of fine-grained OOD detection and indicating room for further improvement.
This observation is further supported by the Risk-Coverage analysis (Figure \ref{fig:risk-coverage}), where Open-Set scenarios consistently exhibit a steeper increase in risk compared to other OOD settings across all datasets.
This indicates that a larger proportion of novel samples are assigned high-confidence in-distribution predictions, reflecting their close proximity to the ID manifold. Such behaviour underscores the intrinsic difficulty of Open-Set detection, where subtle semantic and visual differences make reliable separation fundamentally challenging.

In ID classification, \texttt{MARVEL} consistently achieves the best performance in both overall accuracy and balanced accuracy across all datasets (Table \ref{tab:overall_acc_bacc}).
In contrast, competing methods exhibit a clear trade-off between these metrics.
COCL achieves second-best balanced accuracy on RFMiD and ISIC2019, reflecting its use of logit adjustment to explicitly address class imbalance, but does not attain comparable gains in overall accuracy.
PASCL achieves the second-best accuracy on ISIC2019 but not balanced accuracy, indicating improved feature learning without sufficient correction of class imbalance.
Similarly, PATT achieves the second-best accuracy on RFMiD and NCTCRC, but does not consistently achieve the second-best balanced accuracy, with the exception of NCTCRC where it attains both the second-best accuracy and balanced accuracy.
These observations suggest that achieving strong performance across both metrics is non-trivial in long-tailed settings and depends critically on how decision boundaries are learned under imbalance.
The superior ID classification performance of \texttt{MARVEL} can be attributed to two key factors.
First, the proposed NvMF classifier enables non-linear decision boundaries that better capture complex and overlapping class distributions, which is particularly beneficial for tail classes that lie near decision boundaries. 
Second, the margin-aware multi-expert design explicitly accounts for the heterogeneous nature of long-tailed data by allowing different experts to specialise across various regions of the label distribution (head, mid and tail), thereby reducing head-class dominance and improving representation for underrepresented classes.
Together, these design choices improve class separability while maintaining balanced performance across head and tail categories.

The impact of the design choices on uncertainty estimation is illustrated in Figure \ref{fig:calibration}.
OE and EAT exhibit underconfident predictions across inputs. This behaviour stems from their shared use of strong OOD regularization, which suppresses prediction confidence globally and leads to low-confidence outputs even for correctly classified ID samples.
In contrast, PASCL, PATT, and COCL tend to be overconfident, as their training objectives emphasise feature separation and sharp decision boundaries, resulting in high-confidence predictions even in ambiguous regions.
\texttt{MARVEL}, on the other hand, demonstrates improved calibration due to its multi-expert design, which effectively acts as an implicit ensemble.
In uncertain regions, disagreement among experts leads to moderated confidence estimates, reducing overconfident errors.
As a result, uncertainty is modelled structurally through the architecture rather than relying on post-hoc adjustments.

As for various OOD detection strategies, 
distance-based methods such as KNN and MD perform the worst overall (Table \ref{tab:ablation-ood-detectors}), as they rely on well-formed and compact feature clusters, an assumption that breaks down in long-tailed settings where tail class representations are sparse and poorly defined.
This limitation is particularly pronounced in nearOOD scenarios, where tail samples and OOD inputs occupy similar low-density regions, making distance-based scores unreliable.
In contrast, classifier-based methods such as MSP, MLS, and EBO achieve stronger performance by leveraging decision boundary information, which remains more robust under class imbalance.
Hybrid approaches such as NN-Guide partially mitigate these issues but still fall short of fully classifier-based strategies.
These results suggest that feature space alone is insufficient for reliable OOD detection under imbalance, motivating our choice of MSP as the scoring function.

Table \ref{tab:ablation-classifier-experts} provides further insights into architectural design choices.
For ID classification, hyperspherical classifiers (cosine, vMF, NvMF) consistently outperform a standard FC classifier (Table \ref{tab:ablation-classifier-experts}-A), as feature norms in long-tailed settings tend to correlate with class frequency, leading to biased representations.
Normalising features onto the hypersphere removes this bias and improves angular separation.
Among these, the proposed NvMF classifier achieves the best performance due to its non-linear decision boundaries, which are necessary to separate classes with heterogeneous and overlapping feature distributions prevalent in imbalanced settings.
For OOD classification, removing the outlier expert leads to a consistent drop in performance (Table \ref{tab:ablation-classifier-experts}-B), highlighting the importance of explicit OOD modelling.
The outlier expert operates on a binary classification task that separates ID and OOD samples.
During training, OOD samples from the auxiliary dataset are balanced with ID samples within each batch, resulting in a well-conditioned learning setting. 
Unlike ID classification, this task does not involve class imbalance or complex multi-class decision boundaries. Hence, we use a simple FC classifier to effectively separate the two distributions. Additional modelling capacity offers limited benefit, making the FC classifier a natural and effective choice, with NvMF remaining competitive (Table \ref{tab:ablation-classifier-experts}-B).

From a data-centric perspective, these trends are further influenced by the intrinsic characteristics of the datasets.
RFMiD is the most challenging due to subtle visual differences between retinal pathologies, low inter-class separability, and severe imbalance, particularly for rare conditions.
In contrast, NCTCRC exhibits more distinct structural patterns, enabling clearer feature separation.
ISIC2019 lies between these extremes, with moderate variability across classes but significant intra-class diversity.
These differences highlight that OOD detection performance is closely tied to data complexity, where subtle variations and rare classes increase ambiguity between ID and OOD samples.
An imbalance ratio of 1:50 (tail:head) was chosen to reflect a realistic long-tailed regime. At this imbalance, the least frequent classes have limited samples, e.g., 4 in RFMiD, 72 in ISIC2019, and 170 in NCTCRC, indicating severe sparsity. Higher ratios would further reduce tail-class representation, shifting the problem towards extreme few-shot learning, complicating estimation of decision boundaries and evaluation metrics. Thus, the chosen imbalance ratio simulates realistic distributions and maintains sufficient data for meaningful learning and fair comparison. Under this ratio of severe yet statistically reliable imbalance, the ability of \texttt{MARVEL} to model complex decision boundaries and adapt to label imbalance is critical for achieving consistent performance.

Despite the strong performance of \texttt{MARVEL}, a few limitations remain. 
Firstly, the framework relies on auxiliary OOD data for training the outlier expert, and its effectiveness may depend on the diversity and representativeness of these samples.
Secondly, although our evaluation spans multiple datasets and OOD scenarios, real-world clinical settings may involve even broader and compounded distribution shifts.
Further, a common limitation of OOD methods arises from the assumption of a fixed and well-defined label space. In real-world clinical practice, diagnostic taxonomies are often hierarchical, evolving, and subject to inter-observer variability. For instance, certain conditions may be grouped differently across institutions or redefined as new clinical knowledge emerges. As a result, the distinction between ID and OOD samples may itself become ambiguous, particularly when rare or borderline cases lie between established categories. 
 
Future directions of this work include exploring an even wider range of distribution shifts for various imaging modalities and improving the interpretability of OOD decisions, for instance, by identifying the features or regions that contribute to OOD prediction; recent advances in foundation models and vision-language models offer a promising direction for such explanations. Another interesting extension would be continual or lifelong learning settings, where the framework can incorporate new classes and evolving distributions without degrading performance on previously learned categories. Further work would involve extending the approach to dense prediction tasks, such as segmentation, to broaden its applicability to more complex workflows.

\section{Conclusions}
In this work, we have addressed OOD detection under long-tailed class distributions in medical imaging using \texttt{MARVEL}, an integrated framework that combines a nonlinear vMF classifier, a margin-aware multi-expert learning strategy, and a dedicated outlier expert for explicit OOD modelling. In addition to methodological contributions, our work also includes a clinically motivated evaluation protocol that spans a wide spectrum of distributional shifts, including both nearOOD and farOOD scenarios, across diverse medical imaging datasets.
Experimental results demonstrate consistent improvements over state-of-the-art methods, highlighting the effectiveness of jointly addressing class imbalance and OOD detection in a unified manner.
This capability makes \texttt{MARVEL} well-suited for real-world applications, where reliable identification of anomalous or OOD cases is critical for supporting decision-making and ensuring safety. Future directions include exploring wider distribution shifts, improving the explainability, and extending the framework for dense prediction tasks such as medical image segmentation.

\end{multicols}

\nocite{*} 
\printbibliography

\end{document}